\documentclass{article}
\usepackage{fullpage,graphicx,psfrag,amsmath,amsthm,amsfonts,amssymb,verbatim,algpseudocode,comment,mathtools}
\usepackage{bbm}
\usepackage[linesnumbered,ruled,vlined]{algorithm2e}

\usepackage{natbib} 
\usepackage{graphicx}
\usepackage{xcolor}
\usepackage{enumerate}
\usepackage{textcomp}
\usepackage{caption}
\usepackage{subcaption}
\usepackage{mathrsfs}
\usepackage[citecolor=blue]{hyperref} 
\usepackage{thmtools,thm-restate}                             

\newtheorem{theorem}{Theorem}
\newtheorem{lemma}{Lemma}

\newtheorem{corollary}{Corollary}

\newtheorem{assumption}{Assumption}

\theoremstyle{definition}

\usepackage{epigraph}
\usepackage{etoolbox}

\def\regret{\mathcal{R}}
\def\environment{\mathcal{E}}

\def\pseudoenvironment{\tilde{\environment}}

\def\target{\chi}
\def\pseudotarget{\tilde{\target}}

\def\actions{\mathcal{A}}

\def\histories{\mathcal{H}}

\def\regret{\mathcal{R}}
\def\KL{\mathbf{d}_{\mathrm{KL}}}
\def\TV{\mathbf{d}_{\mathrm{TV}}}

\def\diffentropy{\bf h}

\def\proxytheta{\hat{\theta}}

\def\E{\mathbb{E}}

\def\H{\mathbb{H}}
\def\diffentropy{\mathbf{h}}
\def\I{\mathbb{I}}
\def\Pr{\mathbb{P}}
\def\R{\mathbb{R}}
\def\1{\mathbf{1}}
\def\0{\mathbf{0}}

\def\Var{\mathrm{Var}}

\DeclareMathOperator*{\argmax}{arg\,max}
\DeclareMathOperator*{\argmin}{arg\,min}

\newcommand{\Oc}{\mathcal{O}}
\newcommand\numberthis{\addtocounter{equation}{1}\tag{\theequation}} 

\def\nln{\nonumber \\}
\def\p#1{\left( #1 \right) }
\def\sb#1{\left[ #1 \right] }

\def\bbar {\, \big | \,}
\def\bdbar {\, \big \| \,}
\def\sk#1#2{\stackrel{(#1)}{#2}}

\usepackage{hyperref}

\usepackage[mathscr]{euscript}
\hypersetup{
    colorlinks=true,
    linkcolor=blue,
    filecolor=magenta,      
    urlcolor=purple,
}
\usepackage{xcolor}

\newcount\Comments
\newcommand{\kibitz}[2]{\ifnum\Comments=1{\textcolor{#1}{\textsf{\footnotesize #2}}}\fi}
\usepackage{color}
\definecolor{darkred}{rgb}{0.7,0,0}
\definecolor{darkgreen}{rgb}{0.0,0.5,0.0}
\definecolor{darkblue}{rgb}{0.0,0.0,0.5}
\definecolor{teal}{rgb}{0.0,0.5,0.5}

\title{Gaussian Imagination in Bandit Learning}
\date{Draft Version: \today}
\author{Yueyang Liu 
\\ {\normalsize{yueyl@stanford.edu}}
\and Adithya M. Devraj
\\ {\normalsize{adevraj@stanford.edu}} 
\and Benjamin Van Roy
\\ {\normalsize{bvr@stanford.edu}}
\and Kuang Xu
\\ {\normalsize{kuangxu@stanford.edu}}
}
\date{
\large
Stanford University, Stanford, CA 94305
}

\begin{document}

\maketitle

\abstract{Assuming distributions are Gaussian often facilitates computations that are otherwise intractable. We study the performance of an agent that attains a bounded information ratio with respect to a bandit environment with a Gaussian prior distribution and a Gaussian likelihood function when applied instead to a Bernoulli bandit. Relative to an information-theoretic bound on the Bayesian regret the agent would incur when interacting with the Gaussian bandit, we bound the increase in regret when the agent interacts with the Bernoulli bandit. If the Gaussian prior distribution and likelihood function are sufficiently diffuse, this increase grows at a rate which is at most linear in the square-root of the time horizon, and thus the per-timestep increase vanishes. Our results formalize the folklore that so-called {\it Bayesian agents} remain effective when instantiated with diffuse misspecified distributions.
}
\section{Introduction}
Early in the nineteenth century, Carl Friedrich Gauss studied astronomical observations through the lens of his namesake distribution.  While not perfectly capturing nuances of errors he intended to model, the use of a Gaussian distribution expedited calculations.  Indeed, Gauss was able to determine the orbit of a comet in a single hour where Leonhard Euler required three days \citep{Hall1970,Hermann2006}.  Ever since, pretending that random variables are Gaussian and designing solutions suited to that have served as common and effective practices in modeling and analysis.  In this paper, we study 
implications of these practices---which we refer to as {\it Gaussian imagination}---in the context of bandit learning.

Pretending that distributions are Gaussian can greatly facilitate computations carried out by a bandit learning agent.  For example, so-called {\it Bayesian agents}---such as Thompson sampling \citep{thompson1933likelihood}, information-directed sampling \citep{russo2018learning}, Bayes-UCB \citep{kaufmann2012bayesian}, and the knowledge gradient algorithm \citep{ryzhov2012knowledge}---can be implemented by computing at each time a posterior distribution over mean rewards and then selecting the next action based on this distribution.  In general, the associated computational requirements can be onerous, but they become manageable if relevant distributions -- the prior distribution and likelihood function -- are Gaussian.  A question that arises is whether an agent designed under these distributional assumptions ought to perform well---in terms of accumulating rewards, not just computational efficiency---when these assumptions are inconsistent with beliefs about the true environment.

We assess agent performance in terms of Bayesian regret.
One way of bounding Bayesian regret is through first bounding an agent's information ratio, which is a statistic that quantifies how the agent trades off between regret and information.
Various versions of the information ratio have been proposed and studied over the past decade \citep{RussoIDSNeurips2014,bubdekkorper15,russo2016information,bubeck2016multi,russo2018learning,Dong2018TS,russo2018satisficing,nikkirberkra18,zimlat19,latsze19,luvan19,bubsel,latsze20e,latgyo20,kirlatkra20,lu2021reinforcement,lattimore2021bandit,devraj2021bit}.  Each depends on beliefs about the environment, as expressed by a prior distribution and likelihood function.  Previous results establish that if an agent attains an attractive information ratio with respect to true beliefs, it also attains some level of effectiveness in accumulating rewards.  Consider an agent designed for a Gaussian bandit in the sense that it attains a low information ratio with respect to a Gaussian prior distribution and a Gaussian likelihood function. Will such an agent remain effective when true beliefs about the environment are characterized by different distributions?  Our analysis address this question when the true environment is a Bernoulli bandit. 

We establish a general Bayesian regret bound that applies to {\it any} agent.  In particular, we bound the amount by which Bayesian regret in the Bernoulli bandit can exceed an information-theoretic bound, one that applies when true beliefs are Gaussian.  Interestingly, we establish that this excess grows at a rate which is at most linear in the square-root of the time horizon and therefore represents a vanishing per-timestep difference.  

Our finding represents a dramatic improvement over what existing bounds suggest regarding the cost of misspecification.  For instance, \cite{simchowitz2021bayesian} and \cite{russo2014learning} establish bounds that grow quadratically in time horizon and exponentially in the number of actions, respectively.  To further explore its implications, we specialize our general Bayesian regret bound to 
Thompson sampling and information-directed sampling, each with computations carried out using imaginary Gaussian distributions.  This leads to $\Oc(\actions \sqrt{T \log T})$ bounds on the Bayesian regret incurred by these agents when applied to a Bernoulli bandit with suitably chosen imaginary Gaussian distributions, where $\actions$ and $T$ denote the number of actions and the time horizon.  The optimal bound for the Bernoulli bandit in terms of $\actions$ and $T$ is known to be $\Oc(\sqrt{\actions T})$ \citep{bubeck2013priorfree,latsze19}, which represents a factor of $\Oc(\sqrt{\actions \log T})$ difference.  It remains to be determined whether this difference is fundamental to use of misspecified Gaussian distributions or introduced due to our method of analysis.

A key assumption underlying our analysis is that the misspecified Gaussian distributions are sufficiently diffuse.  The importance of diffuseness has also been highlighted in work on frequentist analysis of Thompson sampling and KL-UCB applied to bandits with independent arms \citep{honda2013optimality, wager2021diffusion, fan2021fragility}.  In contrast, our results apply to {\it any} agent and allow for generalization across arms.  Further, our lens of Bayesian regret draws focus to a {\it true} prior, which is a concept missing from frequentist analysis, allowing us to characterize the impact of prior misspecification.  As such, our results formalize the folklore that Bayesian agents remain effective with misspecified distributions that are sufficiently diffuse.

Our analysis leverages properties of the Gaussian distribution that afford a level of robustness.  In particular, we exploit the fact that posterior covariances evolve in a manner that does not depend on realized rewards except through their influence on subsequent actions.  Covariances encode the agent's uncertainty, and this data-agnostic aspect of their updating ensures that the imaginary learning process reduces uncertainty regardless of the true reward distribution.  Uncertainty guides exploration, and without this reduction, an agent can engage in costly over-exploration.  It is noteworthy that our analysis resides more in the domain of Gauss than Laplace in the sense that it relies on the algebraic properties rather than the ability of the Gaussian distribution to represent asymptotic behavior of random phenomena. Whether Gaussian imagination has deeper connections to the latter remains an interesting question.

\section{Bandit Environments}
In this section, we introduce a general formulation for bandit environments and the concept of regret.  We also establish an information-theoretic regret bound. In particular, we show that if an algorithm satisfies an information ratio bound with respect to the true beliefs, we can derive an upper bound on the regret.  We will subsequently specialize this bandit environment formulation to Bernoulli and Gaussian bandits and study the implications of the regret bound.

The probability framework based on which we develop our analysis is introduced in Appendix \ref{sec:probability}. We also introduce information-theoretic concepts and notations, together with some useful relations in Appendix \ref{sec:information}. 

\subsection{Formulation}

Let $\actions = \{1,\ldots,|\actions|\}$ be a finite set of actions. Let $(R_t: t \in \mathbb{Z}_{++})$ be a random sequence of reward vectors, each taking values in $\R^\actions$. Note that we often use $\actions$ as shorthand for the set cardinality $|\actions|$. We sometimes denote the sequence of reward vectors by $R_{1:\infty}$ and, more generally, a sub-sequence $(R_t, R_{t+1}, \ldots, R_{t'})$ by $R_{t:t'}$. 

We think of rewards as generated by an environment. Formally, an environment $\environment$ is a random probability measure such that $R_{1:\infty}$ is i.i.d.~conditioned on $\environment$, and $\Pr(R_{1} \in \cdot | \environment) = \environment(\cdot)$. More precisely, $\environment$ is a probability measure-valued random variable that
takes on values in the set of all probability measures on $(\R^{\actions}, \mathcal{B})$, where $\mathcal{B}$ is the Borel $\sigma$-algebra on $\R^{\actions}$.
Let $\theta = \E[R_{1} | \environment]$ denote the vector of the mean rewards.
Let $A_* \sim \mathrm{unif}(\argmax_{a \in \actions} \theta_a)$ denote an optimal action. Let $\histories_t$ denote the set comprised of all sequences consisting of $t$ action-reward pairs. Let $\histories = \cup_{t=0}^\infty \histories_t$.  We refer to elements of $\histories$ as {\it histories}.  

We consider an agent that executes an {\it agent policy} $\pi_{\text{agent}}$, which assigns, for each realization of history $ h \in \mathcal{H}$, a probability $\pi_{\text{agent}}(a|h)$ of choosing an action $a$, for all $a \in \actions$. Fixing an arbitrary policy $\pi$, define $H^{\pi}_0$ as the empty history and $H^{\pi}_t = (A^{\pi}_0, R_{1,A^{\pi}_0}, \ldots, A^{\pi}_{t-1}, R_{t,A^{\pi}_{t-1}})$, where $\Pr(A^{\pi}_t = \cdot |H^{\pi}_t) =  \pi(\cdot|H^{\pi}_t)$ for each time $t \in \mathbb{Z}_{+}$. Note that $H^{\pi}_t$ represents the history generated as an agent executes policy $\pi$ by sampling each action $A^{\pi}_t$ from $\pi(\cdot|H^{\pi}_t)$ and receives the resulting reward $R_{t+1,A^{\pi}_t}$.  We denote the infinite sequence of actions and rewards by $H^{\pi}_\infty = (A^{\pi}_0, R_{1,A^{\pi}_0}, \ldots)$. Much of the paper studies properties of interactions under a specific policy $\pi_{\text{agent}}$. When it is clear from context, we suppress superscripts that indicate this. For example, we use $A_t = A^{\pi_{\text{agent}}}_t$, $H_t = H^{\pi_{\text{agent}}}_t$ for all $t \in \mathbb{Z}_+$, and $H_{\infty} = H^{\pi_{\text{agent}}}_{\infty}$ through much of the paper.

Over a horizon $T \in \mathbb{Z}_{+}$, the agent accumulates reward $\sum_{t=0}^{T-1} R_{t+1,A_t}$. We denote the maximum mean reward across actions by $R_* = \max_{a \in \actions} \theta_a$. We study an agent's performance in terms of the regret, short for the cumulative Bayesian regret:
\begin{equation}
\label{eq:regret}
\regret(T) = \E\left[\sum_{t=0}^{T-1} (R_* - R_{t+1,A_t})\right].
\end{equation}

\subsection{Information Ratio}
\label{sec:IR}

Our information-theoretic analysis of bandit learning centers around the concept of an information ratio. A basic version of the information ratio is defined by
$$\Gamma_{\environment} = \sup_{t \in \mathbb{Z}_+, h \in \histories_t} \frac{\E[R_* - R_{t+1, A_t} | H_t = h]^2}{\I(\environment; A_t, R_{t+1, A_t} | H_t = h)}.$$
This ratio represents a tradeoff between the expected regret $\E[R_* - R_{t+1, A_t} | H_t = h]$ incurred over a single timestep and the information $\I(\environment; A_t, R_{t+1, A_t} | H_t = h)$ gained about the environment. A small information ratio reflects the agent's ability to either maintain a low level of regret, gain a large amount of information, or both.

We now consider a more general version of the information ratio, defined with respect to a {\it learning target} $\target$, which is a random variable for which $\target \perp H_\infty | \environment$. This information ratio is defined by
$$\Gamma_\target = \sup_{t \in \mathbb{Z}_+, h \in \histories_t} \frac{\E[R_* - R_{t+1, A_t} | H_t = h]^2}{\I(\target; A_t, R_{t+1, A_t} | H_t = h)}.$$
Intuitively, a learning target represents a collection of statistics about the environment that the agent might aim to learn. As illustrated in Figure \ref{fig:target}, the requirement that $\target \perp H_\infty | \environment$ ensures that all information about the learning target that helps in predicting rewards is present in the environment. The new denominator represents information gained about the learning target $\chi$ rather than the environment $\environment$. It is worth mentioning that the learning target serves as a means to quantify the amount of information the agent accumulates, and is not necessarily the purpose of learning. That is, an agent may or may not aim to explicitly maximize learning with respect to the learning target. 

\begin{figure}[!ht]
\centering
\includegraphics[scale=0.237]{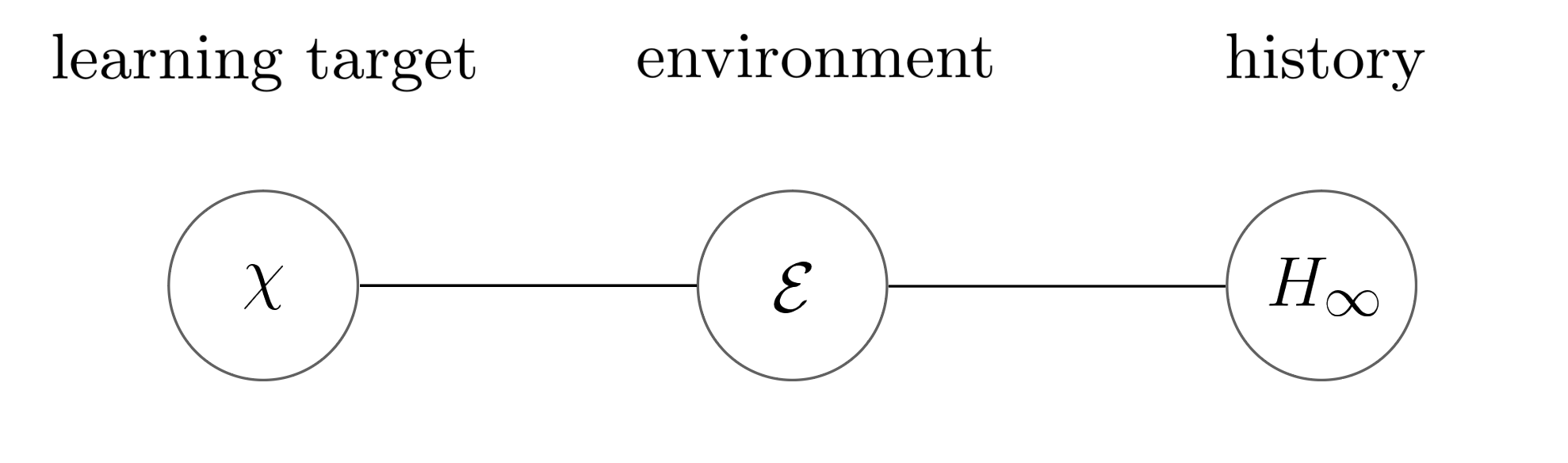}
\caption{Bayesian network expressing dependencies among random variables of interest.} 
\label{fig:target}
\end{figure}

The environment $\environment$ itself represents a possible choice of learning target.  However, the amount of information required to identify the environment, roughly captured by the entropy $\H(\environment)$,
may be intractably large or even infinite.
As such, it is useful to consider alternative learning targets.  A useful learning target is one that can be identified with modest, $ \I(\target; \environment) $ nats of information, where $\I(\target; \environment) \ll \H(\environment)$, while still rich enough to differentiate the mean rewards so as to enable effective action selection.

Over specific realizations of actions and rewards, an agent that learns a particular target $\target$ might not converge on the optimal action $A_*$ and thus on optimal per-timestep reward.  To allow for a meaningful definition of information ratio, we introduce a more general form that incorporates dependence on a tolerance $\epsilon \in \R_+$:
\begin{align*}
\Gamma_{\target,\epsilon} = \sup_{t \in \mathbb{Z}_+, h \in \histories_t} \frac{\E[R_* - R_{t+1, A_t} - \epsilon | H_t = h]_+^2}{\I(\target; A_t, R_{t+1, A_t} | H_t = h)}.
\end{align*}

\subsection{Regret Bound}
\label{sec:gen_reg_bd}

We introduce a result that bounds the regret in terms of the information ratio $\Gamma_{\target, \epsilon}$, the mutual information $\I(\target; \environment)$ between the learning target and the environment, the tolerance $\epsilon$, and the time horizon $T$.
The result is formally stated in the following theorem.
Note that this theorem is not the main result of the paper--- it serves as a foundation based on which we develop our analysis and a baseline against which we compare our main result. Similar information-theoretic bounds have been established in previous studies, a closely related one being \citep{lu2021reinforcement}. 

\begin{restatable}{theorem}{generalregretbound}
\label{th:general-regret-bound}
For all learning target $\target$, tolerance $\epsilon \in \R_+$, 
and time horizon $T \in \mathbb{Z}_{++}$, 
\begin{align*}
\regret(T) \leq \sqrt{\I(\target; \environment) \Gamma_{\target,\epsilon} T} + \epsilon T.
\end{align*}
\end{restatable}
A complete proof of the theorem is presented in Appendix~\ref{sec:general_regret_bound}. Here we provide an outline. 
The proof carries out in two steps. First, the regret is upper-bounded as follows: 
\begin{align*}
\regret(T)
\overset{}{\leq}& \sqrt{\sum_{t=0}^{T-1} \I(\target; A_t, R_{t+1, A_t}|H_t)} \sqrt{\Gamma_{\target, \epsilon} T} + \epsilon T. \numberthis
\label{eq:grb_eq_2_0}
\end{align*}

By the chain rule of mutual information (Lemma~\ref{lemma:chain_MI} in Appendix~\ref{sec:information}), we have
\begin{align*}\sum_{t=0}^{T-1} \I(\target; A_t, R_{t+1, A_t}|H_t) 
= \I(\target; H_T) 
\leq \I(\target; \environment), \numberthis
\label{eq:grb_eq_1_0}
\end{align*}
where the inequality follows from the data processing inequality of mutual information (Lemma~\ref{lemma:dp_MI} in Appendix~\ref{sec:information}) and the fact that $\target$ and $H_{\infty}$ are independent conditioned on $\environment$. 
We establish the desired bound by plugging \eqref{eq:grb_eq_1_0} into \eqref{eq:grb_eq_2_0}. 

The result established by Theorem~\ref{th:general-regret-bound} applies to any bandit environment and any agent policy. The bound depends on the agent policy through the information ratio $\Gamma_{\target, \epsilon}$. Intuitively, the information ratio captures how the agent trades off between regret and information acquired about the learning target. This is multiplied by the amount of information $\I(\target; \environment)$ that must be acquired in order to identify the learning target $\target$.  This product characterizes a decreasing portion of per-timestep regret. In particular, divided by time elapsed, the first term of the bound becomes $\sqrt{\I(\target; \environment) \Gamma_{\target,\epsilon} / T}$, which vanishes as $T$ grows.  The second term $\epsilon T$ reflects a per-timestep regret of at most $\epsilon$ incurred because the agent might only identify the learning target $\target$ rather than the full environment $\environment$.

 Although the regret bound introduced in this section is generally useful, it is not directly applicable to scenarios in which the agent has misspecified beliefs. 
The challenge is to bound the information ratio defined with respect to the true environment, when an agent is designed with misspecified beliefs in mind. In light of this observation, we focus in the next section on developing a new line of analysis that would allow us to bound the true regret, but using an information ratio computed with respect to an imaginary environment that is consistent with the agent's misspecified beliefs. 

\section{Gaussian Imagination}
\label{sec:gaus_ber}
To study misspecification, we specialize the bandit environment formulation introduced in Section 2 to Bernoulli and Gaussian bandit environments, and study the performance of an agent that attains a low information ratio with respect to a Gaussian bandit while interacting with a Bernoulli bandit. Recall that Theorem \ref{th:general-regret-bound} of Section~\ref{sec:gen_reg_bd} can be directly applied when the environment is truly Gaussian. In this section, we develop information-theoretic analysis based on a change-of-measure argument that bounds the amount of regret that exceeds the bound prescribed by Theorem \ref{th:general-regret-bound}.
We establish conditions under which this excess grows at a rate that is at most linear in $\sqrt{T}$, where $T$ is the time horizon. Parallel to Theorem~\ref{th:general-regret-bound}, our regret bound depends on the information ratio with respect to a Gaussian bandit.

\subsection{Formulation}
\label{sec:formulation_ber_gauss}
To distinguish the Bernoulli bandit with which the agent interacts and the Gaussian bandit with respect to which the agent attains a bounded information ratio, we refer to the former as the {\it real environment} $\environment$---or, simply, {\it environment}---and the latter as the {\it imaginary environment} $\pseudoenvironment$, as it exists in the agent's imagination.

Recall that an environment is a \emph{random} probability measure on $\R^\actions$. For a real environment, this probability measure is determined by a random mean reward vector $\theta$, which takes values in $[0,1]^\actions$, and is given by a $\mathrm{Bernoulli}(\theta)$ distribution.  This is a multivariate distribution for which the $a$th component is independently distributed according to $\mathrm{Bernoulli}(\theta_a)$, conditioned on $\theta$. An imaginary environment $\pseudoenvironment$ is similarly determined by a random mean reward vector $\tilde{\theta}$, though in this case the vector takes values in $\R^\actions$, and a positive real number $\sigma^2$. An imaginary environment $\pseudoenvironment$ is a random probability measure described by  $\mathcal{N}(\tilde{\theta}, \sigma^2 I)$, 
where the parameter $\sigma^2 \in \R_{++}$ is fixed, and $\tilde{\theta}$ is distributed according to $ \mathcal{N}(\mu_0, \Sigma_0)$, for some fixed $\mu_0 \in \R^{\actions}$ and $\Sigma_0 \in \mathcal{S}_{++}^{\actions}$, where $\mathcal{S}_{++}^{\actions}$ denotes the set of $\actions \times \actions$ positive definite symmetric matrices.

We denote per-timestep rewards by $R_t$ and $\tilde{R}_t$, referring to them as reward and imaginary reward, optimal actions by $A_*$ and $\tilde{A}_*$, and optimal mean rewards by $R_*$ and $\tilde{R}_*$, respectively. 

For all $t \in \mathbb{Z}_{+}$, we denote by $ H_t = (A_0, R_{1, A_0}, ... , A_{t-1}, R_{t, A_{t-1}})$ the real history generated over $ t $ timesteps by applying the agent policy $\pi_{\text{agent}}$ to the real environment with Bernoulli rewards , and by $\tilde{H}_t = (\tilde{A}_0, \tilde{R}_{1, \tilde{A}_0}, ... , \tilde{A}_{t-1}, \tilde{R}_{t, \tilde{A}_{t-1}})$ the imaginary history generated over $t$ timesteps by applying the agent policy to the imaginary environment with Gaussian rewards. In particular, for all $t \in \mathbb{Z}_{+}$, $\Pr(A_t \in \cdot | H_t) = \pi_{\text{agent}}(\cdot | H_t)$ and $\Pr(\tilde{A}_t \in \cdot | \tilde{H}_t) = \pi_{\text{agent}}(\cdot | \tilde{H}_t)$. Note that when applying $\pi_{\text{agent}}$ to Bernoulli bandits, only real histories are ever used as an input to the algorithm. Nevertheless, the imaginary histories serve as useful conceptual device that help articulate what the agent believes.

Let $\histories_t$ denote the set comprised of all sequences consisting of $t$ action-reward pairs, where the rewards are binary, and let $\histories = \cup_{t=0}^\infty \histories_t$. Let $\tilde{\histories}_t$ denote the set comprised of all sequences consisting of $t$ action-reward pairs, where the rewards are real-valued, and let $\tilde{\histories} = \cup_{t=0}^\infty \tilde{\histories}_t$.

\subsection{Change of Measure}
Our analysis relies on a change of measure. 
As discussed in Section~\ref{sec:formulation_ber_gauss}, for all $t \in \mathbb{Z}_+$, the agent observes a realization of $H_t$ while believing the observation is a realization of $\tilde{H}_t$. To formally describe such behavior and to study the regret under different probability measures, we now introduce some notation used in our analysis. 

Consider random variables $X$ and $Y$ with images $\mathcal{X}$ and $\mathcal{Y}$, respectively, and for all realization $y \in \mathcal{Y}$, a conditional distribution $\Pr(X \in \cdot | Y = y)$. Let $f(y) = \Pr(X \in \cdot |Y=y)$ for $y \in \mathcal{Y}$. For any random variable $Z$ whose image is a subset of $\mathcal{Y}$, $f(Z)$ is a well-defined random variable---it is a probability measure-valued random variable.
We denote by $\Pr(X \in \cdot |Y \leftarrow Z)$ the random variable $f(Z)$. Note that, in general, $\Pr(X \in \cdot |Y = Z) \neq \Pr \left(X \in \cdot \bbar Y \leftarrow Z \right)$ because the former conditions on an event $\{\omega: Y = Z\}$, whereas the latter represents a change of measure from $Y$ to $Z$.  We similarly write $\E[X|Y \leftarrow Z]$ to denote $g(Z)$ for a function $g: \mathcal{Y} \rightarrow \R$ defined by $g(y) = \E[X|Y=y]$ for all $y \in \mathcal{Y}$.  
This notation extends to our information-theoretic concepts.  
For example, if $\mathcal{X}$ and $\mathcal{Y}$ 
are finite, then
\begin{align*}
\H(X | Y \leftarrow Z) = - \sum_{x \in \mathcal{X}}  \Pr(X = x| Y \leftarrow Z) 
 \ln \Pr(X=x | Y \leftarrow Z).
\end{align*} 
Note that this is a random variable due to its dependence on $Z$. Let $U$ be a random variable. 
Analogously, with respect to conditional mutual information, we have 
$$\I(X; U|Y \leftarrow Z) = \H(X | Y \leftarrow Z) - \H(X | U, Y \leftarrow Z),$$
where
\begin{align*}
\H(X | U, Y \leftarrow Z) = 
- \sum_{x \in \mathcal{X}} \Pr(X = x| U, Y \leftarrow Z) 
 \ln \Pr(X=x | U, Y \leftarrow Z).
\end{align*}
Note that $\H(X|Y) = \E[H(X|Y\leftarrow Y)]$ and $\I(X; U|Y) = \E[\I(X; U | Y \leftarrow Y)]$.

In the rest of the paper, expressions involving $\tilde{H}_t \leftarrow H_t$ appear frequently, as the agent observes a realization of $H_t$ while practicing Gaussian imagination, pretending that the observation is sampled from the distribution of $\tilde{H}_t$.

\subsection{Main Theorem}
\label{sec:main_regret_bound}
This section states our main results. We begin by defining a notion of information ratio when the agent observes the true histories while taking them to be the imaginary ones. Following how we defined a learning target with respect to a general bandit environment in Section~\ref{sec:gen_reg_bd}, we define an imaginary learning target $\pseudotarget$ with respect to the imaginary environment $\pseudoenvironment$. We require $\pseudotarget$ be independent of $\tilde{H}_{1 : \infty}$ conditioned on $\pseudoenvironment$. We also define a tolerance $\epsilon \in \R_{+}$. 
Recall that $\tilde{H}_t$ is the imaginary history generated over $t$ timesteps by applying $\pi_{\text{agent}}$ to the imaginary environment, and that $\Pr(\tilde{A}_t \in \cdot | \tilde{H}_t) = \pi_{\text{agent}}(\cdot | \tilde{H}_t)$. 
Let $\Gamma_{\pseudotarget,\epsilon}$ be the information ratio of the agent computed with respect to the imaginary environment while interacting with the real environment:
\begin{align*}
	\Gamma_{\pseudotarget, \epsilon} = \sup_{t \in \mathbb{Z}_{+}, h \in \mathcal{H}_t} \frac{\E\left[\tilde{R}_{*} - \tilde{R}_{t+1, \tilde{A}_t} - \epsilon \bbar  \tilde{H}_t = h\right]^2_{+}}{\I\left(\pseudotarget; \tilde{A}_t, \tilde{R}_{t+1, \tilde{A}_t} \bbar  \tilde{H}_t = h\right)}.
\end{align*}
 Note that since $\histories_t \subset \tilde{\histories}_t$ for all $t \in \mathbb{Z}_+$, $\Gamma_{\pseudotarget, \epsilon} $ is trivially upper-bounded by the following information ratio $\tilde{\Gamma}_{\pseudotarget, \epsilon}$, where the maximization takes place over the set $  \tilde{\histories}_t $: 
\begin{align*}
    \tilde\Gamma_{\pseudotarget, \epsilon} = \sup_{t \in \mathbb{Z}_{+}, h \in \tilde{\mathcal{H}}_t} \frac{\E\left[\tilde{R}_{*} - \tilde{R}_{t+1, \tilde{A}_t} - \epsilon \bbar  \tilde{H}_t = h\right]^2_{+}}{\I\left(\pseudotarget; \tilde{A}_t, \tilde{R}_{t+1, \tilde{A}_t} \bbar  \tilde{H}_t = h\right)}.
\end{align*}
Note that $ \tilde\Gamma_{\pseudotarget, \epsilon} $ is simply the information ratio associated with the same agent, when interacting with a Gaussian bandit described by $\pseudoenvironment$. For this reason, we refer to this information ratio $    \tilde\Gamma_{\pseudotarget, \epsilon} $ as the  
\emph{Gaussian information ratio}. 
In this section, we establish an upper bound on the regret through upper-bounding the Gaussian information ratio.

Two assumptions form the basis of the main results. Assumption \ref{ass:optimism} (Optimism) requires the imaginary environment be such that the agent's imagined expected optimal rewards is greater than that of the actual, as such encouraging exploration. Crucially, as shown in an example in Section~\ref{sec:opt_ass}, with a large class of real environments, a sufficient condition for Assumption~\ref{ass:optimism} (Optimism) to hold is for the Gaussian prior distribution and likelihood function to be diffuse, i.e. to have sufficiently large variances.
\begin{assumption} 
	\label{ass:optimism}
	(Optimism) For all $t \in \mathbb{Z}_+$,
	 	\begin{equation*}
		\mathbb{E}\left[\mathbb{E}\left[\tilde{R}_{*} \bbar  \tilde{H}_t \leftarrow H_t\right]\right] \geq \mathbb{E}[R_*].
	\end{equation*}
\end{assumption}

 Assumption \ref{ass:gaussian} (Gaussianity) is a condition on the imaginary environment and the imaginary learning target. The assumption ensures that various key statistics in the imaginary environment admit Gaussian posterior distributions. This in turn allows us to analytically characterize the agent's learning progress in the imaginary environment, a key ingredient in the final regret bound.
 Section~\ref{sec:Sufficient_Conditions} discusses the assumptions in greater detail, and provides examples in which the assumptions hold.
 
\begin{assumption}
\label{ass:gaussian}
 (Gaussianity) The imaginary learning target $\pseudotarget$ and the imaginary mean reward $\tilde{\theta}$ are jointly Gaussian. 
\end{assumption}

Under the two assumptions, we establish the main result of the paper in the form of a regret bound.
\begin{restatable}{theorem}{bernoullitogaussian}
\label{th:bernoulli_to_Gaussian}
Let $\pseudoenvironment$ be an imaginary environment, 
$\pseudotarget$ an imaginary learning target, and $\epsilon \in \R_+$ a tolerance. 
Suppose Assumptions~\ref{ass:optimism} and \ref{ass:gaussian} hold. For all time horizon $T \in \mathbb{Z}_{++}$, the regret of an agent interacting with a Bernoulli bandit environment $\environment$ satisfies 
\begin{align*}\label{eq:key}
	\regret(T) \leq \sqrt{\I\left(\pseudotarget ; \pseudoenvironment \right) \tilde{\Gamma}_{\pseudotarget, \epsilon} T} + \epsilon T 
	+ \gamma \sqrt{2\KL \p{\Pr\Big(\theta \in \cdot\Big) \bdbar \Pr\left(\tilde{\theta} \in \cdot\right)} T},
\end{align*}
where 
\begin{equation}
    \gamma = \sup_{a \in \mathcal{A}, t \in \mathbb{Z}_{+}, h_t \in \mathcal{H}_t} \mathbb{E}\left[\tilde{\theta}_a \bbar  \tilde{H}_t = h_t\right].
    \nonumber 
\end{equation}
\end{restatable}

The bound depends on the policy through the information ratio $\tilde{\Gamma}_{\pseudotarget, \epsilon}$ with respect to the Gaussian bandit. Recall that Theorem~\ref{th:general-regret-bound} of Section~\ref{sec:gen_reg_bd} suggests that the regret of any agent interacting with a Gaussian bandit can be bounded if the agent attains a small information ratio with respect to the Gaussian bandit, and that the regret bound established by Theorem~\ref{th:general-regret-bound} is exactly the sum of the first two terms in the regret bound established by Theorem~\ref{th:bernoulli_to_Gaussian}. So under Assumptions~\ref{ass:optimism} and \ref{ass:gaussian}, bounding the information ratio $\tilde{\Gamma}_{\pseudotarget, \epsilon}$ would simultaneously bound the agent's regret incurred in a Bernoulli and a Gaussian bandit, respectively. 

Notably, the third term in the regret bound established by Theorem~\ref{th:bernoulli_to_Gaussian} captures the excess regret as the result of the agent's beliefs differing from those of the real environment. This excess is determined by $\gamma$ and the KL-divergence between the mean rewards $\theta$ and the imaginary mean rewards $\tilde{\theta}$, and grows at a rate which is at most linear in $\sqrt{T}$. 
Section~\ref{sec:bounding_gamma} presents a sufficient condition under which $\gamma$ is favorably bounded.

\subsection{Examples Where the Assumptions Hold}
\label{sec:Sufficient_Conditions}
This section provides examples in which Assumption~\ref{ass:optimism} (Optimism) and Assumptions~\ref{ass:gaussian} (Gaussianity) hold, and $\gamma$, which appears in the regret bound established by Theorem \ref{th:bernoulli_to_Gaussian} of Section~\ref{sec:main_regret_bound}, is conveniently bounded.

\subsubsection{Assumption~\ref{ass:optimism} (Optimism)}
\label{sec:opt_ass}
Below is a sufficient condition for Assumption~\ref{ass:optimism} (Optimism): For all $ t \in \mathbb{Z}_+$, $h \in \mathcal{H}_t$,
 	\begin{align*}
 	\label{eq:opt-suff}
		\mathbb{E}\left[\tilde{R}_{*} \bbar  \tilde{H}_t = h\right] \geq \mathbb{E}[R_* | H_t = h]. 
	\end{align*}
This sufficient condition ensures that 
$\mathbb{E}[\tilde{R}_{*} \bbar  \tilde{H}_t \leftarrow H_t] \geq \mathbb{E}[R_* | H_t = H_t]$ a.s., which is much stronger than Assumption~\ref{ass:optimism} (Optimism) as it requires that the conditional expected optimal rewards in the agent's imagination to be larger than that of the actual on all sample paths. 

Following is an example, in the form of a lemma, in which the sufficient condition holds. The proof of the lemma centers around arguments developed in Section~6.5 in \citep{JMLR:v20:18-339} and is given in Appendix~\ref{sec:ass_opt}. Note that in this example, we require the real environment $\environment$ to satisfy certain conditions. 

\begin{restatable}{lemma}{gaussianbetadominance}
\label{lemma:gaussian_beta_dominance}
 Suppose the following hold:
(i) $\alpha_a + \beta_a \geq 3$ for all $a \in \actions$,
(ii) $\theta_a \sim \mathrm{Beta}(\alpha_a, \beta_a)$, independently across all $a \in \actions$. Let $\pseudoenvironment$ be an imaginary environment such that: 
(iii) ${\sigma}^2 \geq 3$,
(iv) $\Sigma_{0}$ is diagonal with $\Sigma_{0, a, a} \geq \frac{{\sigma}^2}{\alpha_a + \beta_a}$ for all $a \in \actions$,
(v) $\mu_{0, a} \geq \frac{\alpha_a}{\sigma^2} \Sigma_{0, a, a}$, for all $a \in \actions$. We have for all $ t \in \mathbb{Z}_+$ and $h \in \mathcal{H}_t$, 
 	\begin{align*}
		\mathbb{E}\left[\tilde{R}_{*} \bbar  \tilde{H}_t = h\right] \geq \mathbb{E}[R_* | H_t = h].
	\end{align*}
\end{restatable}

\subsubsection{Assumption~\ref{ass:gaussian} (Gaussianity)}
\label{sec:Gaussian-target}

We describe one particular imaginary learning target, where $\tilde{\theta}$ can be thought of as a noisy perturbation of the learning target, and this target satisfies Assumption~\ref{ass:gaussian} (Gaussianity). We also compute $\I(\pseudotarget, \pseudoenvironment)$, the amount of information in the imaginary environment needed to identify $\pseudotarget$. 

Recall that the imaginary Gaussian bandit environment is determined by an $\actions$-dimensional vector $\tilde{\theta}$.  We consider a learning target, another $\actions$-dimensional vector $\pseudotarget = \hat{\theta}$, \emph{defined with respect to $\tilde{\theta}$ and a perturbation variance $\delta^2 \in (0, 1)$}. Fix a $\delta^2 \in (0, 1)$, let $\hat{\theta}$ be defined such that
$\Pr(\tilde{\theta} - \hat{\theta} \in \cdot|\hat{\theta}) \sim \mathcal{N}\left(0,\delta^2 \Sigma_0 \right)$. We use $Z$ to denote $\tilde{\theta} - \hat{\theta}$. Note that, while $Z$ is independent of $\hat{\theta}$, it is generally not independent of $\theta$.  As such, we have constructed the learning target $\chi = \hat{\theta}$ such that the mean reward vector can be viewed as a noisy perturbation $\tilde{\theta} = \hat{\theta} + Z$ of the learning target. 

With this definition of the learning target, Assumption \ref{ass:gaussian} (Gaussianity) is trivially satisfied as stated in the following lemma. The proof is straightforward and is presented in Appendix~\ref{sec:gaussianity_condition}. 
\begin{restatable}{lemma}{gaussianitycondition}
\label{lemma:gaussianity_condition}
For all $\delta^2 \in (0, 1)$, 
if $\hat{\theta}$ is a learning target defined with respect to $\tilde{\theta}$ and perturbation variance $\delta^2$, then $\hat{\theta}$ and $\tilde{\theta}$
are jointly Gaussian. 
\end{restatable}

The following result characterizes the mutual information between the imaginary learning target and the imaginary environment.
\begin{restatable}{lemma}{gaussianmutualinformation}
\label{lemma:gaussian_mutual_information}
For all $\delta^2 \in (0, 1)$, if $\hat{\theta}$ is a learning target defined with respect to $\tilde{\theta}$ and perturbation variance $\delta^2$, then the mutual information between $\hat{\theta}$ and the imaginary environment $\pseudoenvironment$ is given by
\vspace{-1em}
\begin{align*}
\I\left(\hat{\theta}; \pseudoenvironment \right) = \frac{\actions}{2}  \ln \left( \frac{1}{\delta^2 } \right).
\end{align*}
\label{le:mi_theta_hat_theta}
\end{restatable}
\vspace{-1em}
Recall that an optimal action in the imaginary environment is given by $\tilde{A}_* \in \argmax_{a \in \actions} \tilde{\theta}_a$ and results in expected imaginary reward $\tilde{R}_* = \tilde{\theta}_{A_*}$. The vector $\hat{\theta}$ can be interpreted as an approximation of $\tilde{\theta}$ and it provides the expected imaginary reward conditioned on the learning target: $ \E[\tilde{R}_{t+1} | \hat{\theta}] = \E[\tilde{\theta} | \hat{\theta}] = \E[\hat{\theta} + Z | \hat{\theta}] = \hat{\theta}$. An agent with knowledge of $\hat{\theta}$ but not $\tilde{\theta}$ might consider selecting an action $\hat{A}$ by sampling from $\mathrm{unif}(\argmax_{a \in \actions} \hat{\theta}_a)$ and enjoy reward $\hat{R} = \tilde{\theta}_{\hat{A}}$ instead of $\tilde{R}_*$.  The expected difference $\E[\tilde{R}_* - \hat{R}]$ represents a loss due to this use of $\hat{\theta}$ instead of $\tilde{\theta}$.  On one hand, as the perturbation variance $\delta^2$ increases, the mutual information $\I(\pseudotarget; \pseudoenvironment)$ decreases, controlling the information about the environment required to identify the learning target.  On the other hand, increasing $\delta$ also increases the loss $\E[\tilde{R}_* - \hat{R}]$.

\subsubsection{Bounding $\gamma$}
\label{sec:bounding_gamma}
Recall that $\gamma = \sup_{a \in \mathcal{A}, t \in \mathbb{Z}_{+}, h_t \in \mathcal{H}_t} \mathbb{E}[\tilde{\theta}_a | \tilde{H}_t = h_t]$ appears in the regret bound established by Theorem~\ref{th:bernoulli_to_Gaussian} in Section~\ref{sec:main_regret_bound}. To obtain favorable regret bounds by applying Theorem~\ref{th:bernoulli_to_Gaussian}, we would like to establish conditions under which $\gamma$ is small. Below we provide a such example.  

We say that a $n \times n$ matrix $A$ is diagonally dominant if $\min_{1 \leq i \leq n} (|A_{ii}| - \sum_{j \neq i} |A_{ij}|) \geq 0$, and $A$ is strictly diagonally dominant if the inequality is strict. Recall that the mean reward $\tilde{\theta}$ in the imaginary environment is distributed according to a multivariate Gaussian distribution $\tilde{\theta} \sim \mathcal{N}(\mu_0, \Sigma_0)$, where $\mu_0 \in \R^{\actions}$ and $\Sigma_0 \in \mathcal{S}^{\actions}_{++}$ are fixed. 
The following lemma, the proof of which is given in Appendix~\ref{sec:gamma_bound}, describes a sufficient condition under which $\gamma$ is bounded and provides a bound on $\gamma$.

\begin{restatable}{lemma}{boundedgamma}
\label{lemma:diagonally_dominant}
 If $\Sigma_0^{-1}$ is strictly diagonally dominant, then $\gamma \leq 2 \max_{a \in \actions} |\mu_{0, a}| + 1$.
\end{restatable}
Note that for independent-armed bandit, $\Sigma_0$ is diagonal with positive diagonal entries.  
As such, its inverse $\Sigma_0^{-1}$ is also diagonal with positive entries along its diagonal, and hence is strictly diagonally dominant. So independent-armed bandits serve as examples in which $\gamma$ can be conveniently bounded. If in addition the prior mean $\mu_0 \in [0, 1]^{\actions}$, we obtain a bound of $\gamma \leq 3$.

\section{Examples}
\label{sec:examples}
By Theorem~\ref{th:bernoulli_to_Gaussian} of Section~\ref{sec:main_regret_bound}, the regret of any agent interacting with a Bernoulli bandit can be bounded in terms of the information ratio $\tilde{\Gamma}_{\pseudotarget, \epsilon}$ with respect to a Gaussian bandit. This section demonstrates that we can bound the Gaussian information ratio $\tilde{\Gamma}_{\pseudotarget, \epsilon}$ for particular agents, where Gaussian imagination renders otherwise intractable computations efficient. Thompson sampling and information-directed sampling agents serve as such examples. When these agents are designed for Gaussian prior distributions and likelihood functions, we refer to them as \emph{Gaussian Thompson sampling} and \emph{Gaussian information-directed sampling} agents.

Algorithms that implement either of these agents typically maintain a posterior distribution $\Pr(\theta \in \cdot | H_t)$ over mean rewards.  For a Bernoulli bandit with an arbitrary prior distribution, this can be computationally intractable.  However, with Gaussian imagination, the posterior distribution $\Pr(\tilde{\theta} \in \cdot | \tilde{H}_t \leftarrow H_t)$ of the imaginary mean reward vector $\tilde{\theta}$ can be efficiently maintained.  In particular, this distribution is Gaussian, with mean vector $\mu_t$ and covariance matrix $\Sigma_t$ determined by $H_t$.  The posterior parameters $\mu_t$ and $\Sigma_t$ can be updated incrementally upon each observation according to
\begin{align*}
\mu_{t+1} \!=\! &\ \left(\Sigma_t^{-1} \!+\! \frac{1}{{\sigma}^2} \1_{A_t} \1_{A_t}^\top \right)^{-1} \! \left(\Sigma_t^{-1} \mu_t \!+\! \frac{1}{{\sigma}^2} \1_{A_t} R_{t+1,A_t} \right)\!,\\
\Sigma_{t+1} \!=\! &\ \left(\Sigma_t^{-1} + \frac{1}{{\sigma}^2} \1_{A_t} \1_{A_t}^\top\right)^{-1}.
\end{align*}

We establish that the information ratios of Gaussian Thompson sampling and Gaussian information-directed sampling are bounded according to $\tilde{\Gamma}_{\pseudotarget, \epsilon} \leq 2 \actions \sigma^2$, for suitably chosen learning target $\pseudotarget$ and tolerance $\epsilon$. Applying Theorem~\ref{th:general-regret-bound} of Section~\ref{sec:gen_reg_bd} and Theorem~\ref{th:bernoulli_to_Gaussian} of Section~\ref{sec:main_regret_bound} result in $\Oc(\actions \sqrt{T \log T})$ bounds on the regret incurred by either agent in a Bernoulli or Gaussian environment.

\subsection{Gaussian Thompson Sampling}
\label{sec:Gaussian_Bernoulli_TS}
We introduce a Gaussian Thompson sampling agent that executes a policy $\pi^{\mathrm{TS}}$ that selects each action via Thompson sampling assuming the imaginary environment $\pseudoenvironment$. In particular, the agent samples action $A_t$ from 
\begin{align*}\pi^{TS}(\cdot |H_t) = \Pr\left(\tilde{A}_* \in \cdot \bbar \tilde{H}_t \leftarrow H_t\right),
\end{align*}
which is the posterior distribution of the optimal action $\tilde{A}_*$ of the imaginary environment $\pseudoenvironment$, pretending the real history is the imaginary one. 
\subsection{Gaussian Information-Directed Sampling}
\label{sec:GausBer-IDS}

Consider the regret bound in Theorem~\ref{th:general-regret-bound}. For a given learning target $\target$ and a tolerance  $\epsilon \in \R_+$, the agent policy impacts this regret bound via the corresponding information ratio ${\Gamma}_{\target, \epsilon}$. 
This in turn enables the possibility that to minimize the regret bound, the agent policy should minimize the corresponding information ratio. Minimizing the information ratio at each time-step produces a policy, the information-directed sampling, that strikes an effective balance between exploration and exploitation (for more discussion, see Section~6.3 of \citep{DBLP:journals/corr/abs-2103-04047}).

We introduce a Gaussian information-directed sampling agent that executes a policy $\pi^{\mathrm{IDS}}$ that selects each action via information-directed sampling assuming the imaginary environment $\pseudoenvironment$. After fixing a learning target $\pseudotarget = \tilde{\theta}$,
the Gaussian information-directed sampling agent aims to minimize the corresponding information ratio $\tilde{\Gamma}_{\tilde{\theta}}$, pretending the real history is the imaginary one. In particular, for all $t \in \mathbb{Z}_+$, Gaussian information-directed sampling selects action $A_t \sim \pi^{\mathrm{IDS}}(\cdot|H_t)$,
where
\begin{equation}
\pi^{\mathrm{IDS}}(\cdot|H_t) = \argmin_{\pi \in \Delta_\actions} \frac{\E\left[\tilde{R}_{*} - \tilde{R}_{t+1, A_t} | \tilde{H}_t \leftarrow H_t\right]^2}{\I\left(\tilde{\theta}; A_t, \tilde{R}_{t+1, A_t} | \tilde{H}_t \leftarrow H_t\right)},
\label{eq:satisficing-IDS_Gauss-Bern}
\end{equation}
where $\Delta_\actions$ is all probability distributions over $\actions$.

\subsection{Bounding the Information Ratio}

\label{sec:bound_info_ratio}
We derive an upper bound on the information ratio $\tilde{\Gamma}_{\pseudotarget, \epsilon}$ associated with $\pi_{\text{agent}}$ with respect to the Gaussian bandits, for suitably chosen learning target $\pseudotarget$ and   tolerance $\epsilon$. The bound is formally stated in the following lemma, a complete proof of which is given in Appendix~\ref{sec:gaussian_ir_bound}.

\begin{restatable}{lemma}{gaussianir}
\label{lemma:IR_bound}
For all $\delta \in (0, 1)$, if 
$\hat{\theta}$ is the learning target defined with respect to $\tilde{\theta}$ and tolerance variance $\delta^2$ and that $\epsilon = \delta \sqrt{\actions \max_{a \in \actions} \Sigma_{0, a, a}}$, then
the information ratios of Gaussian Thompson sampling and Gaussian information-directed sampling with respect to the Gaussian bandit environment satisfy
\begin{align*}
\tilde{\Gamma}_{\hat{\theta},\epsilon} \leq 2 \actions \sigma^2.
\end{align*}
\end{restatable}

The above lemma suggests that for all $\delta \in (0, 1)$, we can then define learning target $\hat{\theta}$ and tolerance $\epsilon$ such that the information ratio $\tilde{\Gamma}_{\hat{\theta}, \epsilon}$ is bounded. Note that Assumption~\ref{ass:gaussian} (Gaussianity) is satisfied.

\subsection{Regret Bounds}
\label{sec:regret_bounds}
Now that we have successfully bounded the Gaussian information ratio $\tilde{\Gamma}_{\hat{\theta}, \epsilon}$, we are ready to upper-bound the regret bound established by applying Theorem~\ref{th:general-regret-bound} of Section~\ref{sec:gen_reg_bd} to Gaussian bandits, which is also the sum of the first two terms in the regret bound established by Theorem~\ref{th:bernoulli_to_Gaussian} of Section~\ref{sec:main_regret_bound}.  Applying Lemma~\ref{lemma:IR_bound} of Section~\ref{sec:bound_info_ratio} and optimizing over $\delta \in (0, 1)$, we can derive the following theorem, the proof of which is given in Appendix~\ref{sec:reg_bds}.

\begin{restatable}{theorem}{regretbounds}
\label{th:Gaussian_regret}
For all $T \in \mathbb{Z}_{++}$, 
there exists tolerance $\epsilon \in \mathbb{R}_+$ and learning target $\hat{\theta}$ defined with respect to $\tilde{\theta}$ and some perturbation variance $\delta^2 \in (0, 1)$ such that 
Gaussian Thompson sampling and Gaussian information-directed sampling satisfies
\begin{align*}
\sqrt{\I\left(\hat{\theta}; \pseudoenvironment \right) \tilde{\Gamma}_{\hat{\theta},\epsilon}T} + \epsilon T \leq \sigma \actions \sqrt{T \ln \left( \frac{T}{\actions} \right) } + 
\actions \sqrt{T \max_{a \in \actions} \Sigma_{0, a, a}}.
\end{align*}
\end{restatable}

As direct corollaries, we establish regret bounds for Gaussian Thompson sampling and Gaussian information-directed sampling agents interacting with a Gaussian bandit or a Bernoulli bandit, by applying Theorem~\ref{th:general-regret-bound} of Section~\ref{sec:gen_reg_bd} and Theorem~\ref{th:bernoulli_to_Gaussian} of Section~\ref{sec:main_regret_bound},  respectively.

\begin{corollary} For all $T \in \mathbb{Z}_{++}$,
the regret of a Gaussian Thompson sampling agent or a Gaussian information-directed sampling agent interacting with a Gaussian bandit environment $\pseudoenvironment$ satisfies 
\begin{align*}
\regret(T) \leq \sigma \actions \sqrt{T \ln \left( \frac{T}{\actions} \right) } + 
\actions \sqrt{T \max_{a \in \actions} \Sigma_{0, a, a}}
\end{align*}
\end{corollary}

\begin{corollary}
\label{cor:mis}
Under Assumption~\ref{ass:optimism}, for all $T \in \mathbb{Z}_{++}$,
the regret of a Gaussian Thompson sampling agent or a Gaussian information-directed sampling agent interacting with a Bernoulli bandit environment $\environment$ satisfies 
\begin{align*}
\regret(T) \leq \sigma \actions \sqrt{T \ln \left( \frac{T}{\actions} \right) } + 
\actions \sqrt{T \max_{a \in \actions} \Sigma_{0, a, a}} + \gamma \sqrt{2 T \KL \p{\Pr(\theta \in \cdot) \bdbar \Pr\left(\tilde{\theta} \in \cdot\right)}}.
\end{align*}
where $
\gamma = \sup_{a \in \mathcal{A}, t \in \mathbb{Z}_{+}, h_t \in \mathcal{H}_t} \mathbb{E}\left[\tilde{\theta}_a \bbar  \tilde{H}_t = h_t\right]
$.
\end{corollary}
Recall that $\hat{\theta}$ and $\tilde{\theta}$ are jointly Gaussian by Lemma~\ref{ass:gaussian} of Section~\ref{sec:Sufficient_Conditions}, so Assumption~\ref{ass:gaussian} (Gaussianity) automatically holds. Then the regret bound in Corollary~\ref{cor:mis} holds under Assumption~\ref{ass:optimism} (Optimism). This suggests that when the Gaussian prior distribution and likelihood function are sufficiently diffuse, the regret of applying Gaussian Thompson sampling or Gaussian information-directed sampling to Bernoulli bandits can be upper-bounded by $\Oc(\actions \sqrt{T \log T})$. 

The optimal known Bayesian regret bound for Bernoulli bandits is $\Oc(\sqrt{\actions T})$ \citep{bubeck2013priorfree,latsze19} — suggesting a factor of $\Oc(\sqrt{\actions \log T})$ difference between the bounds. It remains to be determined whether this difference is fundamental to use of misspecified Gaussian distributions or introduced due to our method of analysis.

\section{Proof of Theorem \ref{th:bernoulli_to_Gaussian}}
\label{sec:proof_main_theorem}

In this section, we provide a proof to Theorem \ref{th:bernoulli_to_Gaussian}. We restate the theorem below for ease of referencing. 
\bernoullitogaussian*
The proof carries out in two steps, and we present the main steps in the form of theorems. The first theorem shows
that under Assumption~\ref{ass:optimism} (Optimism), the regret of the agent can be  bounded from above by a sum of two terms, the \emph{imaginary regret} and a term that involves the KL-divergence between $\theta$ and $\tilde{\theta}$. The second theorem bounds the \emph{imaginary regret}. 

\subsection{Decomposing the Regret}
We first introduce the following result by decomposing the regret. The first term in the theorem is what we refer to as the 
\emph{imaginary regret}. It is named so, because the summand $\E[ \E[\tilde{R}_* - \tilde{R}_{t+1,\tilde{A}_t} \bbar  \tilde{H}_t \leftarrow H_t]]$ represents the conditional expected regret in the next timestep in the agent's imagination, when the agent observes a real history while taking it to be an imaginary one. Importantly, this imaginary regret is not to be confused with $\sum_{t = 0}^{T-1} \E[\tilde{R}_* - \tilde{R}_{t+1,\tilde{A}_t}] = \sum_{t = 0}^{T-1}\E[\E[\tilde{R}_* - \tilde{R}_{t+1,\tilde{A}_t} \bbar  \tilde{H}_t ]] $, the regret the agent would incur when interacting with the imaginary, rather than real, environment. The second term in the regret bound established by Theorem~\ref{th:Ber_KL} upper bounds the discrepancy between the true regret and the imaginary regret using the KL-divergence between the prior distributions of the mean rewards and the imaginary mean rewards. 
\begin{restatable}{theorem}{theoremberkl}
	\label{th:Ber_KL} 
	Let $\pseudoenvironment$ be an imaginary environment. 
Suppose Assumption~\ref{ass:optimism} holds. For all time horizon $T \in \mathbb{Z}_{++}$, the regret of an agent interacting with a Bernoulli bandit environment $\environment$ satisfies:
\begin{align*}
\mathcal{R}(T) \leq  \underbrace{\sum_{t = 0}^{T-1} \E\left[\E\left[\tilde{R}_* - \tilde{R}_{t+1,\tilde{A}_t} \bbar  \tilde{H}_t \leftarrow H_t\right]\right]}_{\text{imaginary regret}}
+  \gamma \sqrt{2 \KL\left(\Pr\Big(\theta \in \cdot\Big) \bdbar  \Pr\left(\tilde{\theta} \in \cdot\right)\right) T},
\end{align*}
where 
$\gamma = \sup_{a \in \mathcal{A}, t \in \mathbb{Z}_{+}, h_t \in \mathcal{H}_t} \mathbb{E}\left[\tilde{\theta}_a \bbar  \tilde{H}_t = h_t\right]$. 
\end{restatable}

\begin{proof}  We would like to show that the discrepancy between the real and perceive regret can be written as the sum of two kinds of approximation errors: one between the optimal rewards $R_*$ and $\tilde R_*$, and the other between the per-step rewards $R_{t+1, A_t}$ and $\tilde{R}_{t+1,\tilde{A}_t}$. In the remainder of the proof, we will demonstrate how to analyze these two loss terms and show that they are bounded from above by a quantity that depends on the KL-divergence between the real and the pseudo environments.  

 We first decompose the one-step real regret into three one-step loss terms. For all $t \in \mathbb{Z}_+$:
\begin{align} 
&\ \E[R_* - R_{t+1,A_t}] \nln
=&\  \E\left[\E\left[\tilde{R}_* - \tilde{R}_{t+1,\tilde{A}_t} \bbar  \tilde{H}_t \leftarrow H_t\right] + \E\left[R_* - R_{t+1,A_t} | H_t\right] - \E\left[\tilde{R}_* - \tilde{R}_{t+1,\tilde{A}_t} \bbar  \tilde{H}_t \leftarrow H_t\right]\right] \nln
= &\ {\E\left[\E\left[\tilde{R}_* - \tilde{R}_{t+1,\tilde{A}_t} \bbar  \tilde{H}_t \leftarrow H_t\right]\right]}
+  \E\left[R_*\right]  - \E\left[\E\left[\tilde{R}_* \bbar  \tilde{H}_t \leftarrow H_t\right]\right]   + \nln 
& \  \E\left[\E\left[\tilde{R}_{t+1,\tilde{A}_t} \bbar \tilde{H}_t \leftarrow H_t\right]  - \E\left[R_{t+1,A_t} | H_t\right]\right] \nln
\leq &\ \underbrace{\E\left[\E\left[\tilde{R}_* - \tilde{R}_{t+1,\tilde{A}_t} \bbar  \tilde{H}_t \leftarrow H_t\right]\right]}_{\text{imaginary regret}} + \underbrace{\E\left[\E\left[\tilde{R}_{t+1,\tilde{A}_t} \bbar \tilde{H}_t \leftarrow H_t\right]  - \E\left[R_{t+1,A_t} | H_t\right]\right]}_{\text{one-step approximation loss}} 
\label{eq:decomposition}
\end{align}
where the inequality in the last step is based on the fact that the approximation error between the optimal rewards, $ \E\left[R_*\right]  - \E\left[\E\left[\tilde{R}_* \bbar  \tilde{H}_t \leftarrow H_t\right]\right]$, is always non-positive by Assumption \ref{ass:optimism} (Optimism).

Next, we would like to develop an upper bound on the {one-step approximation loss}, as a function of the KL-divergence between the real and the pseudo-rewards. Before presenting the formal proof, we will first explain the high-level strategy. 
First, we will bound the approximation loss using the total variational distance  between $ R_{t+1, A_t} $ and $ \tilde R_{t+1, \tilde{A}_t}$, by making use of the following fact: 
\begin{restatable}{lemmma}{lemmaexpectationtv}
	\label{lemma:expecatation_tv} Fix $ B \in \R_{++} $. Suppose $X$ and $Y$ are two discrete random variables taking values in a discrete alphabet $\mathcal{X} \subset [0, B]$. We have that: 
	\begin{align}
		\left| \E[X] - \E[Y] \right|  \leq B \TV(\Pr(X \in \cdot) \bdbar \Pr(Y \in \cdot)). 
		\label{eq:exp_to_TV}
	\end{align}
\end{restatable}
The proof of the lemma is given in Appendix~\ref{sec:information}.
In light of \eqref{eq:exp_to_TV}, if the imaginary reward $ \tilde{R}$ had shared the same image as the Bernoulli reward $ R $ (which implies that we can let $B = 1$), we would have obtained the following bound: 
\begin{equation}\label{eq:exp_to_tv_approx_ideal}
\E\sb{\E\left[\tilde{R}_{t+1,\tilde{A}_t} \bbar   \tilde{H}_t \leftarrow H_t\right] - \E\left[R_{t+1,A_t} | H_t\right] } \leq \E \sb{ \TV\left(\Pr\Big(R_{t+1,A_{t}} \in \cdot | H_t\Big) \bdbar  \Pr\left(\tilde{R}_{t+1,\tilde{A}_{t}} \in \cdot | H^{\dagger}_t \leftarrow H_t\right)\right)}, 
\end{equation}
Next, using Pinsker's inequality (Lemma \ref{lemma:Pinsker} in Appendix~\ref{sec:information}), we can further express the above bound in terms of KL-divergence: 
\begin{equation}\label{eq:exp_to_KL_approx_ideal}
	\E\sb{\E\left[\tilde{R}_{t+1,\tilde{A}_t} \bbar   \tilde{H}_t \leftarrow H_t\right] - \E\left[R_{t+1,A_t} | H_t\right] } \leq  \E \sb{ \sqrt{ \frac{1}{2}\KL\left(\Pr\Big(R_{t+1,A_{t}} \in \cdot | H_t\Big) \bdbar  \Pr\left(\tilde{R}_{t+1,\tilde{A}_{t}} \in \cdot | H^{\dagger}_t \leftarrow H_t\right)\right)} }. 
\end{equation}
Equation \eqref{eq:exp_to_KL_approx_ideal} will serve as an important building block, from which we can sum over all $ t $ and use the chain rule for KL-divergence to arrive at the final desirable bound on the cumulative approximation loss. 

Our formal proof will follow the above-mentioned outline to arrive at a version of \eqref{eq:exp_to_KL_approx_ideal}. But there are two technical issues that need to be resolved, both of which are caused by the imaginary environment being Gaussian:
\begin{enumerate}
	\item First, the posterior of $ \tilde{R}_{t,a} $ is Gaussian and does not admit a bounded support. Therefore, we cannot directly apply Lemma \ref{lemma:expecatation_tv} to obtain \eqref{eq:exp_to_tv_approx_ideal}.  
	\item Second, the true rewards have a discrete distribution and the imaginary rewards a continuous one. Therefore, the  KL-divergence between the two is infinite and the bound obtained via \eqref{eq:exp_to_KL_approx_ideal} would have been vacuous. 
\end{enumerate}
To circumvent these problems, we will introduce a coupling argument that leverages a sequence of auxiliary rewards $R^{\dagger}_{1:\infty}$ and actions $A^{\dagger}_{0:\infty}$. These auxiliary rewards and actions will be constructed to be coupled to, and thus mimic, those in the imaginary environment, with the crucial difference being that the auxiliary rewards will take values in a finite alphabet. As such, these variables will serve as intermediaries between the real and imaginary environments and allow us to obtain a meaningful bound on the approximation loss in terms of the total-variation distance and  KL-divergence. The following lemma summarizes the useful properties of the auxiliary variables. The proof  is given in Appendix \ref{app:lemma:rdagger}, and provides an explicit construction for these variables. 
\begin{restatable}{lemmma}{lemmardaggerproperties}
\label{lemma:rdagger}
Let $\gamma = \sup_{a \in \mathcal{A}, t \in \mathbb{Z}_{+}, h_t \in \mathcal{H}_t} \mathbb{E}\left[\tilde{\theta}_a \bbar  \tilde{H}_t = h_t\right]$. Then, 
		we can construct, for all $t \in \mathbb{Z}_+$, auxiliary reward vector $R^{\dagger}_{t}$ taking values in a discrete subset of $\R^{\actions}$, and actions $A^{\dagger}_{t}$ taking values in $\actions$ such that the following holds:
    \begin{enumerate} [$ (i) $]
    \item For all $a \in \mathcal{A}$, $R^{\dagger}_{t,a}$ has strictly positive probability mass on $\{0, 1\}$, and $0 \leq R^{\dagger}_{t, a} \leq 2 \gamma$. 
    \item Define $H^\dagger_t = (A^\dagger_0, R^\dagger_{0, A^\dagger_0}, \ldots, A^\dagger_{t-1}, R^\dagger_{t-t, A^\dagger})$. Then, $ \Pr(A_t^{\dagger} \in \cdot \bbar  H^{\dagger}_t) = \pi(\cdot \bbar  H^{\dagger}_t) $. 
    \item For all $a \in \mathcal{A}$, if $\tilde{\theta}_a \in [0, 1]$, then conditional on $\tilde{\theta}_a$, $R^{\dagger}_{t,a}$ is distributed according to $\text{Bernoulli}(\tilde{\theta}_a)$ and independent of the rest of the system. 
    \item Almost surely, 
    \begin{align} 
    \E\left[\tilde{R}_{t + 1, \tilde{A}_t} \bbar  \tilde{H}_t \leftarrow H_t\right] \leq \E\left[{R}^{\dagger}_{t + 1, {A}^{\dagger}_t} \bbar  {H}^{\dagger}_t \leftarrow H_t\right].
    \label{eq:dagger_exp}
\end{align}
\end{enumerate}
\end{restatable}
Note that by definition, $\gamma \geq 1$. 

We are now ready to establish a bound that's analogous to \eqref{eq:exp_to_KL_approx_ideal}, this time using auxiliary rewards and actions: 
\begin{align}
	& \  \E\left[\E\left[\tilde{R}_{t+1,\tilde{A}_t} \bbar   \tilde{H}_t \leftarrow H_t\right] - \E\left[R_{t+1,A_t} | H_t\right]\right]  \nln
	\sk{a}{\leq} &\  \E\left[ \E\left[{R}^{\dagger}_{t+1, {A}^{\dagger}_t} \bbar   H^{\dagger}_t \leftarrow H_t\right] - \E\left[R_{t+1, A_t}| H_t\right]\right] \nln
	\sk{b}{\leq}  &\ 2 \gamma \E\left[ \TV\left(\Pr\Big(R_{t+1,A_{t}} \in \cdot | H_t\Big) \bdbar  \Pr\left(R^{\dagger}_{t+1,A^{\dagger}_{t}} \in \cdot | H^{\dagger}_t \leftarrow H_t\right)\right) \right] \nln
	\sk{c}{\leq} & \ \E\left[ \gamma\sqrt{2 \KL\left(\Pr\Big(R_{t+1,A_{t}} \in \cdot | H_t\Big) \bdbar  \Pr\left(R^{\dagger}_{t+1,A^{\dagger}_{t}} \in \cdot \bbar  H^{\dagger}_t \leftarrow H_t\right)\right)} \right], 
	\label{eq:one-step_approx_to_KL}
\end{align}
where step $ (a) $ follows from property $ (iv) $
of Lemma \ref{lemma:rdagger}, step $ (b) $ from property $ (i) $ of Lemma \ref{lemma:rdagger} and Lemma \ref{lemma:expecatation_tv}, and step $ (c) $ from Pinsker's inequality (Lemma \ref{lemma:Pinsker} in Appendix~\ref{sec:information}). 

To complete the proof, we will leverage properties of the KL-divergence to extend the bound on the one-step approximation loss in \eqref{eq:one-step_approx_to_KL} to one that applies to the cumulative approximation loss. Summing both sides of \eqref{eq:one-step_approx_to_KL}  across $ t $, we obtain 
\begin{align*} \label{eq:cumulative_regret}
&\ \sum_{t = 0}^{T - 1} \E\left[\E\left[\tilde{R}_{t+1,\tilde{A}_t} \bbar   \tilde{H}_t \leftarrow H_t\right] - \E\Big[R_{t+1,A_t} \bbar  H_t\Big]\right]\\
\leq &  \  \gamma\sqrt{2} \E\left[\sum_{t = 0}^{T - 1} \sqrt{ \KL\left(\Pr\left(R_{t+1,A_{t}} \in \cdot \bbar  H_t\right) \bdbar  \Pr\left(R^{\dagger}_{t+1,A^{\dagger}_{t}} \in \cdot \bbar  H^{\dagger}_t \leftarrow H_t\right)\right)} \right]\\
\leq &\  \gamma \sqrt{2T} \p{\E\left[\sum_{t=0}^{T-1} \KL\left(\Pr\left(R_{t+1, A_t} \in \cdot \bbar  H_t\right) \bdbar  \Pr\left(R^{\dagger}_{t+1, A^{\dagger}_t} \in \cdot \bbar  {H}^{\dagger}_t \leftarrow H_t\right)\right)\right]}^{1/2},  \numberthis
\end{align*}
where the last step is based on the Cauchy-Schwartz Inequality.  Applying the chain rule for KL-divergence (Corollary \ref{cor:KL_chainrule}), we  obtain: 
\begin{align*} \label{eq:kl_chain_rule}
&\ \E\left[\sum_{t=0}^{T-1} \KL\left(\Pr\left(R_{t+1, A_t} \in \cdot \bbar  H_t\right) \bdbar  \Pr\left(R^{\dagger}_{t+1, A^{\dagger}_t} \in \cdot \bbar  {H}^{\dagger}_t \leftarrow H_t\right)\right)\right]\\
\leq &\ \E\left[\sum_{t=0}^{T-1} \KL\left(\Pr\left(\Big(A_t, R_{t+1, A_t}\Big) \in \cdot \bbar  H_t\right) \bdbar \Pr\left(\left(A^{\dagger}_t, R^{\dagger}_{t+1, A^{\dagger}_t}\right) \in \cdot \bbar  {H}^{\dagger}_t \leftarrow H_t\right)\right)\right] \\
=&\ \KL\left(\Pr\Big(H_T \in \cdot\Big) \bdbar \Pr\left(H^{\dagger}_T \in \cdot\right)\right). \numberthis
\end{align*}
Next, we will use the  data-processing inequality for KL-divergence (Lemma \ref{lemma:KL_dataprocessing} in Appendix~\ref{sec:information}) to argue that the KL-divergence between real and auxiliary histories can be bounded by the KL-divergence between the real and imaginary environments. Crucially, the following inequality exploits the fact that the probability law that generates the auxiliary history from the imaginary environment $ \tilde \theta $ is the same as the one that generates the history from the real environment $ \theta $. 
\begin{align*} \label{eq:kl_dp}
 \KL\left(\Pr\Big(H_T \in \cdot\Big) \bdbar \Pr\left(H^{\dagger}_T \in \cdot\right)\right) 
\stackrel{(a)}{\leq} &\ \KL\left(\Pr\Big(R_{1:T} \in \cdot\Big) \bdbar \Pr\left(R^{\dagger}_{1:T} \in \cdot\right)\right) \\
\stackrel{(b)}{\leq} &\ \KL\left(\Pr\Big(\theta \in \cdot\Big) \bdbar \Pr\left(\tilde{\theta} \in \cdot\right)\right). \numberthis
\end{align*}
For step $ (a) $, we note that, by property $ (ii) $ of Lemma \ref{lemma:rdagger}, actions in both the real  and auxiliary environments are driven by the same policy, $ \pi $, i.e., $\Pr(A_t^{\dagger} \in \cdot \bbar  H^{\dagger}_t = h)= \Pr( A_t \in \cdot \bbar  {H}_t = h)  = \pi(\cdot \bbar h)$. As a result, one can generate $ H_T $ from $ R_{1:T} $ using the same conditional probability law as $ H^\dagger_T $ from $ R^\dagger_{1:T} $. Step $ (a) $ therefore follows from the data-processing inequality. Step $ (b) $ is based on a similar logic: by property $ (iii)  $ of the construction of $ R^\dagger $ in Lemma \ref{lemma:rdagger},  the sequence $ R_{1:T} $ is generated from $ \theta $ using the same conditional probability law as  $ R^\dagger_{1:T} $ from $ \tilde{\theta} $. And the data-processing again applies. 

Substituting (\ref{eq:kl_chain_rule}) and (\ref{eq:kl_dp}) into  (\ref{eq:cumulative_regret}), we obtain the following bound on the  cumulative approximation loss: 
\begin{align*}
    \sum_{t = 0}^{T - 1} \E\left[\E\left[\tilde{R}_{t+1,\tilde{A}_t} \bbar   \tilde{H}_t \leftarrow H_t\right] - \E\Big[R_{t+1,A_t} \bbar  H_t\Big]\right]
    \leq  \gamma\sqrt{2 \KL\left(\Pr\Big(\theta \in \cdot\Big) \bdbar \Pr\left(\tilde{\theta} \in \cdot\right)\right) T}.
\end{align*}
\noindent Recall that we bounded the cumulative regret by the sum of the agent loss and the cumulative  approximation loss in (\ref{eq:decomposition}). Therefore, 
\begin{equation*}
\mathcal{R}(T) \leq  \sum_{t = 0}^{T-1} \E\left[\E\left[\tilde{R}_* - \tilde{R}_{t+1,\tilde{A}_t} \bbar  \tilde{H}_t \leftarrow H_t\right]\right] + \gamma\sqrt{2 \KL\left(\Pr\Big(\theta \in \cdot\Big) \bdbar \Pr\left(\tilde{\theta} \in \cdot\right)\right) T}.
\end{equation*}
This completes the proof of Theorem \ref{th:Ber_KL}. 
\end{proof}

\subsection{Bounding the Imaginary Regret}
This section introduces the following theorem, which, combined with Theorem~\ref{th:Ber_KL}, proves Theorem~\ref{th:bernoulli_to_Gaussian}. The theorem upper-bounds the imaginary regret by an expression that involves the Gaussian information ratio $\tilde{\Gamma}_{\pseudotarget, \epsilon}$ and the mutual information $\I(\pseudotarget; \pseudoenvironment)$ between the imaginary target and the imaginary environment. Note that the regret bound established by  Theorem~\ref{th:imaginary_regret} appears to be an upper bound on the regret the agent would incur interacting with a Gaussian bandit described by the imaginary environment, by applying Theorem~\ref{th:general-regret-bound} of Section \ref{sec:gen_reg_bd}. 
But there is a crucial and subtle distinction:  Theorem~\ref{th:imaginary_regret} concerns bounding the \emph{imaginary regret} when the agent interacts with a Bernoulli bandit, instead of the regret the agent incurs when interacting with a Gaussian bandit. Theorem~\ref{th:general-regret-bound} does not directly apply and it calls for a new line of analysis. In our analysis, the imaginary environment being Gaussian plays a crucial role. 
The proof relies on the property of Gaussian random variables that the shape of the posterior distribution depends on the data only through the number of samples. 
This allows us to obtain a meaningful bound on the imaginary regret even when the data-generating process differs from that of the imaginary environment.

\begin{theorem}
	\label{th:imaginary_regret}
Let $\pseudoenvironment$ be an imaginary environment, 
$\pseudotarget$ an imaginary learning target, and $\epsilon \in \R_+$ a tolerance. 
Suppose Assumption~\ref{ass:gaussian} holds. For all time horizon $T \in \mathbb{Z}_{++}$, the imaginary regret satisfies:
	\begin{align*}
		\sum_{t = 0}^{T-1} \E\left[\E\left[\tilde{R}_* - \tilde{R}_{t+1,\tilde{A}_t} \bbar  \tilde{H}_t \leftarrow H_t\right]\right]  
		\leq  \sqrt{\I\left(\pseudotarget ; \pseudoenvironment\right) \tilde{\Gamma}_{\pseudotarget, \epsilon} T} + \epsilon T. 
	\end{align*}
\end{theorem}

\proof
We begin by bounding the imaginary regret in a manner similar to steps taken in the proof of Theorem \ref{th:general-regret-bound} for the general setting: 
\begin{align}
	&     \sum_{t = 0}^{T-1} \E\left[\E\left[\tilde{R}_{*} - \tilde{R}_{t + 1, \tilde{A}_t} \bbar  \tilde{H}_t \leftarrow H_t\right]\right] \nln
	\leq &\ \sum_{t = 0}^{T-1} \E\left[\E\left[\tilde{R}_{*} - \tilde{R}_{t + 1, \tilde{A}_t} - \epsilon \bbar  \tilde{H}_t \leftarrow H_t\right]_{+}\right] + \epsilon T \nln
	= &\  \sum_{t = 0}^{T-1} \E\left[\left\{\frac{\E\left[\tilde{R}_{*} - \tilde{R}_{t + 1, \tilde{A}_t} - \epsilon \bbar  \tilde{H}_t \leftarrow H_t\right]_{+}^2}{\I\left(\pseudotarget; \tilde{A}_t, \tilde{R}_{t+1, \tilde{A}_t} \bbar  \tilde{H}_t \leftarrow H_t\right)} \I\left(\pseudotarget; \tilde{A}_t, \tilde{R}_{t+1, \tilde{A}_t} \bbar  \tilde{H}_t \leftarrow H_t\right) \right\}^{1/2}\right]+ \epsilon T \nln
		\stackrel{(a)}{\leq}  &\ \E\left[\sum_{t = 0}^{T-1}\sqrt{\tilde{\Gamma}_{\pseudotarget, \epsilon} \I\left(\pseudotarget; \tilde{A}_t, \tilde{R}_{t+1, \tilde{A}_t} \bbar  \tilde{H}_t \leftarrow H_t\right)}\right] + \epsilon T\nln
	\stackrel{(b)}{\leq} &\ \sqrt{\tilde{\Gamma}_{\pseudotarget, \epsilon}T \E\left[\sum_{t=0}^{T-1} \I\left(\pseudotarget; \tilde{A}_t, \tilde{R}_{t+1, \tilde{A}_t} \bbar  \tilde{H}_t \leftarrow H_t\right)\right]} + \epsilon T
	\label{eq:perc_infogain}
\end{align}
where step $ (a) $ follows from the definition of the information ratio, and step $ (b) $ from the  Cauchy-Schwartz inequality.  
 
The above inequality shows the cumulative imaginary regret can be bounded from above by a function of the information ratio and $ \E\left[\sum_{t=0}^{T-1} \I(\pseudotarget; \tilde{A}_t, \tilde{R}_{t+1, \tilde{A}_t} \bbar  \tilde{H}_t \leftarrow H_t)\right] $, the latter of which term as the imaginary information gain. Note that the history used to calculate the conditional mutual information in the imaginary information gain is generated by the real environment. As such, we cannot directly evoke the chain rule of mutual information to characterize the sum of the one-step imaginary information gains.  Nevertheless, we show in the following lemma that the imaginary information gain is still bounded by the mutual information between the learning target $ \pseudotarget $ and the mean reward vector in the pseudo environment. Crucially, the proof of this result exploits the fact that the imaginary environment is Gaussian, and consequently the imaginary information gain does not depend on the values of the (imaginary)rewards. The proof will deferred till the end of this section. 

\begin{lemma} \label{lemma:imaginary_chainrule} Under Assumption \ref{ass:gaussian}, we have that 
	\begin{align}
		\E\left[\sum_{t=0}^{T-1}\I\left(\pseudotarget; \tilde{A}_t, \tilde{R}_{t+1, \tilde{A}_t} \bbar  \tilde{H}_t \leftarrow H_t\right)\right]
		\leq \I\left(\pseudotarget; \pseudoenvironment \right).
		\label{eq:imaginary_chainrule}
	\end{align}
\end{lemma}

Combining \eqref{eq:imaginary_chainrule} and \eqref{eq:perc_infogain}, we have that 
\begin{align}
\sum_{t = 0}^{T-1} \E\left[\E\left[\tilde{R}_{*} - \tilde{R}_{t + 1, \tilde{A}_t} \bbar  \tilde{H}_t \leftarrow H_t\right]\right] 
	\leq & \ 		  \sqrt{\tilde{\Gamma}_{\pseudotarget, \epsilon}T \E\left[\sum_{t=0}^{T-1} \I\left(\pseudotarget; \tilde{A}_t, \tilde{R}_{t+1, \tilde{A}_t} \bbar  \tilde{H}_t \leftarrow H_t\right)\right]} + \epsilon T  \nln
	\leq & \ \sqrt{\I\left(\pseudotarget ; \pseudoenvironment \right) {\tilde{\Gamma}}_{\pseudotarget, \epsilon} T} + \epsilon T. 
\end{align}
This completes the proof of Theorem \ref{th:imaginary_regret}.  \qed

\proof[Proof of Lemma \ref{lemma:imaginary_chainrule}] We would like to show that there is a variant of the chain rule for mutual information holds for the one-step expected information gain, such that the cumulative expected information gain is upper-bounded by a information theoretic quantity. In particular, we will show that the one-step information gain is
\begin{align} 
   \E\left[\I\left(\pseudotarget; \tilde{A}_t, \tilde{R}_{t+1, \tilde{A}_t} \bbar  \tilde{H}_t \leftarrow H_t\right)\right] = \E\left[\I\left(\pseudotarget; \tilde{\theta} \bbar  \tilde{H}_t \leftarrow H_t\right)\right] - \E\left[\I\left(\pseudotarget; \tilde{\theta} \bbar  \tilde{H}_{t+1} \leftarrow H_{t+1}\right)\right]. 
   \label{eq:variant_mi}
\end{align}
Suppose the above equality holds. We can then employ a telescopic sum across $ t $ to upper-bound the total imaginary information gain as follows. 
\begin{align*} 
	\E\left[\sum_{t=0}^{T-1}\I\left(\pseudotarget; \tilde{A}_t, \tilde{R}_{t+1, \tilde{A}_t} \bbar  \tilde{H}_t \leftarrow H_t\right)\right]
	=&\ \sum_{t = 0}^{T-1} \left(\E\left[\I\left(\pseudotarget; \tilde{\theta} \bbar  \tilde{H}_t \leftarrow H_t\right)\right] - \E\left[\I\left(\pseudotarget; \tilde{\theta} \bbar  \tilde{H}_{t+1} \leftarrow H_{t+1}\right)\right]\right) \\
	=&\ \E\left[\I\left(\pseudotarget; \tilde{\theta} \bbar  \tilde{H}_0 \leftarrow H_0\right)\right] - \E\left[\I\left(\pseudotarget; \tilde{\theta} \bbar  \tilde{H}_{T} \leftarrow H_{T}\right)\right] \\
	=&\ \I\left(\pseudotarget; \tilde{\theta}\right) - \E\left[\I\left(\pseudotarget; \tilde{\theta} | \tilde{H}_{T} \leftarrow H_{T}\right)\right] \\
	\leq &\ \I\left(\pseudotarget; \tilde{\theta}\right) \\
	= &\ \I\left(\pseudotarget; \pseudoenvironment \right).
\end{align*}
This proves the Lemma \ref{lemma:imaginary_chainrule}.

What remains is to prove \eqref{eq:variant_mi}. To this end, we start by expressing the information gain as the difference between conditional mutual information between the learning target and the imaginary environment. For all $t \in \mathbb{Z}_+$ and $h_t \in \mathcal{H}_t$, 
\begin{align*}
&\ \I\left(\pseudotarget; \tilde{A}_t, \tilde{R}_{t+1, \tilde{A}_t} \bbar  \tilde{H}_t = h_t\right) \\
= &\ \diffentropy \left(\pseudotarget \bbar  \tilde{H}_t = h_t\right) - \diffentropy \left(\pseudotarget \bbar  \tilde{H}_t = h_t, \tilde{A}_t, \tilde{R}_{t+1, \tilde{A}_t}\right) \\
= &\ \diffentropy \left(\pseudotarget \bbar  \tilde{H}_t = h_t\right) - \diffentropy \left(\pseudotarget | \tilde{\theta}\right) + \diffentropy \left(\pseudotarget \bbar  \tilde{\theta}\right) - \diffentropy \left(\pseudotarget \bbar  \tilde{H}_t = h_t, \tilde{A}_t, \tilde{R}_{t+1, \tilde{A}_t}\right) \\
\stackrel{(a)}{=} &\ \diffentropy \left(\pseudotarget \bbar  \tilde{H}_t = h_t \right) - \diffentropy \left(\pseudotarget \bbar  \tilde{\theta}, \tilde{H}_t = h_t\right) + \diffentropy \left(\pseudotarget \bbar  \tilde{\theta}, \tilde{H}_t = h_t, \tilde{A}_t, \tilde{R}_{t+1, \tilde{A}_t}\right) - \diffentropy \left(\pseudotarget \bbar  \tilde{H}_t = h_t, \tilde{A}_t, \tilde{R}_{t+1, \tilde{A}_t}\right) \\
= &\ \I\left(\pseudotarget; \tilde{\theta}\Big| \tilde{H}_t = h_t\right)  - \I\left(\pseudotarget; \tilde{\theta} \bbar  \tilde{H}_t = h_t, \tilde{A}_t, \tilde{R}_{t+1, \tilde{A}_t}\right),
\end{align*}
where  step $ (a) $ follows from the independence between $\pseudotarget$ and $(\tilde{H}_t, \tilde{A}_t, \tilde{R}_{t+1, \tilde{A}_t})$ conditional on the realization of $\tilde{\theta}$. Taking expectation on both sides with $ h_t $ replaced by $ H_t $, it follows that for all $t \in \mathbb{Z}_+$, we have that
\begin{align} \label{eq:variant_mi_1}
\E\left[\I\left(\pseudotarget; \tilde{A}_t, \tilde{R}_{t+1, \tilde{A}_t} \bbar  \tilde{H}_t \leftarrow H_t\right)\right]
=\E\left[\I\left(\pseudotarget; \tilde{\theta} \bbar  \tilde{H}_t \leftarrow H_t\right)\right] - \E\left[\I\left(\pseudotarget; \tilde{\theta} \bbar  \tilde{H}_t \leftarrow H_t, \tilde{A}_t, \tilde{R}_{t+1, \tilde{A}_t}\right)\right].
\end{align}

We now focus on the second term of (\ref{eq:variant_mi_1}), and show that we can remove the conditioning on $ \tilde A_t $ and $ \tilde R_{t+1, \tilde{A}_t} $. To that end, the second term of (\ref{eq:variant_mi_1}) can be written as: 
\begin{align*}
& \E\left[\I\left(\pseudotarget; \tilde{\theta} \bbar  \tilde{H}_t \leftarrow H_t, \tilde{A}_t, \tilde{R}_{t+1, \tilde{A}_t}\right)\right] \nln
\stackrel{}{=}&\ \E\left[\sum_{a \in \mathcal{A}} \pi(a | H_t) \int \Pr\left(\tilde{R}_{t+1, \tilde{A}_t} \in dr \bbar \tilde{H}_t \leftarrow H_t, \tilde{A}_t = a\right) \I\left(\pseudotarget; \tilde{\theta} \bbar  \tilde{H}_t \leftarrow H_t, \tilde{A}_t = a, \tilde{R}_{t+1,\tilde{A}_t} = r\right)  \right].
\end{align*}

By  Lemma \ref{lemma:bayes_chi} in Appendix \ref{app:Gaussian_Prop}, conditioned on history, the vector constructed by stacking $\pseudotarget$ and $\tilde{\theta}$ is jointly Gaussian and the covariance matrix depends on the history through the actions. Note that the mutual information between two random variables remains the same if we shift the means. 
Hence, we conclude that the conditional mutual information $\I\left(\pseudotarget; \tilde{\theta} \bbar  \tilde{H}_t \leftarrow H_t, \tilde{A}_t = a, \tilde{R}_{t+1,\tilde{A}_t} = r\right)$ does not depend on the value of the realization, $r$. That is 
\begin{equation}\label{eq:MI_data_indep}
	\I\left(\pseudotarget; \tilde{\theta} \bbar  \tilde{H}_t \leftarrow H_t, \tilde{A}_t = a, \tilde{R}_{t+1,\tilde{A}_t} = r\right) = \I\left(\pseudotarget; \tilde{\theta} \bbar  \tilde{H}_t \leftarrow H_t, \tilde{A}_t = a, \tilde{R}_{t+1,\tilde{A}_t} = r'\right), \quad \forall r, r'.
\end{equation}
We have
\begin{align*} \label{eq:variant_mi_2}
& \E\left[\I\left(\pseudotarget; \tilde{\theta} \bbar  \tilde{H}_t \leftarrow H_t, \tilde{A}_t, \tilde{R}_{t+1, \tilde{A}_t}\right)\right] \nln
\stackrel{}{=}&\ \E\left[\sum_{a \in \mathcal{A}} \pi(a | H_t) \int \Pr\left(\tilde{R}_{t+1, \tilde{A}_t} \in dr | \tilde{H}_t \leftarrow H_t, \tilde{A}_t = a\right) \I\left(\pseudotarget; \tilde{\theta} \bbar  \tilde{H}_t \leftarrow H_t, \tilde{A}_t = a, \tilde{R}_{t+1,\tilde{A}_t} = r\right) \right] \\
\stackrel{(a)}{=} &\ \E\left[\sum_{a \in \mathcal{A}} \pi(a | H_t) \sum_{r \in \{0, 1\}} \Pr\left({R_{t + 1, A_t} = r | H_t, A_t = a}\right) \I\left(\pseudotarget; \tilde{\theta} \bbar  \tilde{H}_t \leftarrow H_t, \tilde{A}_t = a, \tilde{R}_{t+1,\tilde{A}_t} = r\right)\right] \\
\stackrel{}{=} &\ \E\left[\I\left(\pseudotarget; \tilde{\theta} \bbar  \tilde{H}_t \leftarrow H_t, \tilde{A}_t \leftarrow A_t, \tilde{R}_{t+1,\tilde{A}_t} \leftarrow R_{t+1, A_t}\right)\right] \\
=&\ \E\left[\I\left(\pseudotarget; \tilde{\theta} \bbar  \tilde{H}_{t+1} \leftarrow H_{t+1}\right)\right], \numberthis
\end{align*}
where step $ (a) $ follows from \eqref{eq:MI_data_indep}.  This proves (\ref{eq:variant_mi}), and by consequence, Lemma \ref{lemma:imaginary_chainrule}. \qed 

\section{Conclusion}
We establish a regret bound that applies to any agent that attains a favorable information ratio with respect to a Gaussian bandit.  The bound ensures effective behavior when the agent interacts with a Bernoulli bandit.  This result indicates that any learning algorithm designed for a Gaussian bandit with a sufficiently diffuse prior distribution and likelihood function exhibits a high degree of robustness to misspecification.  Hence, the practice of Gaussian imagination---pretending that random variables are Gaussian and designing solutions suited to that---in bandit learning can be justified by rigorous mathematical analysis.

\section*{Acknowledgements}
Financial support from Army Research Office (ARO) grant W911NF2010055 is gratefully acknowledged.

\bibliography{bibtex}
\bibliographystyle{apalike}

\appendix
\renewcommand{\thesection}{\Alph{section}}

\section{Probabilistic Framework}
\label{sec:probability}
Probability theory emerges from an intuitive set of axioms, and this paper builds on that foundation.  Statements and arguments we present have precise meaning within the framework of probability theory.  However, we often leave out measure-theoretic formalities for the sake of readability.  It should be easy for a mathematically-oriented reader to fill in these gaps.

We will define all random quantities with respect to a probability space $(\Omega, \mathcal{F}, \Pr)$.  The probability of an event $F \in \mathcal{F}$ is denoted by $\Pr(F)$.  For any events $F, G \in \mathcal{F}$ with $\Pr(G) > 0$, the probability of $F$ conditioned on $G$ is denoted by $\Pr(F | G)$.

A random variable is a function with the set of outcomes $\Omega$ as its domain.  For any random variable $Z$, $\Pr(Z \in \mathcal{Z})$ denotes the probability of the event that $Z$ lies within a set $\mathcal{Z}$.  The probability $\Pr(F | Z = z)$ is of the event $F$ conditioned on the event $Z = z$.  When $Z$ takes values in $\R$ and has a density $p_Z$, though $\Pr(Z=z)=0$ for all $z$, conditional probabilities $\Pr(F| Z=z)$ are well-defined and denoted by $\Pr(F | Z = z)$.
For fixed $F$, this is a function of $z$.  We denote the value, evaluated at $z=Z$, by $\Pr(F | Z)$, which is itself a random variable.  Even when $\Pr(F | Z = z)$ is ill-defined for some $z$, $\Pr(F | Z)$ is well-defined because problematic events occur with zero probability.  

For each possible realization $z$, the probability $\Pr(Z=z)$ that $Z = z$ is a function of $z$.  We denote the value of this function evaluated at $Z$ by $\Pr(Z)$.  Note that $\Pr(Z)$ is itself a random variable because it depends on $Z$.  For random variables $Y$ and $Z$ and possible realizations $y$ and $z$, the probability $\Pr(Y=y|Z=z)$ that $Y=y$ conditioned on $Z=z$ is a function of $(y, z)$.  Evaluating this function at $(Y,Z)$ yields a random variable, which we denote by $\Pr(Y|Z)$.

Particular random variables appear routinely throughout the paper.  One is the environment $\environment$, a
\emph{random} probability measure over $\R^\actions$ such that, for all $t \in \mathbb{Z}_+$, $\Pr(R_{t+1} \in \cdot | \environment) = \environment(\cdot)$ and $R_{1:\infty}$ is i.i.d.~conditioned on $\environment$. We often consider probabilities $\Pr(F|\environment)$ of events $F$ conditioned on the environment $\environment$.

A policy $\pi$ assigns a probability $\pi(a|h)$ to each action $a$ for each history $h$.  For each policy $\pi$, random variables $A_0^\pi, R_{1,A_0^\pi}, A_1^\pi, R_{2, A_1^\pi}, \ldots$, represent a sequence of interactions generated by selecting actions according to $\pi$.  In particular, with $H_t^\pi = (A_0^\pi, R_{1, A_0^\pi}, \ldots, R_{t, A_{t-1}^\pi})$ denoting the history of interactions through time $t$, we have $\Pr(A^\pi_t|H^\pi_t) = \pi(A^\pi_t|H^\pi_t)$. 
As shorthand, we generally suppress the superscript $\pi$ and instead indicate the policy through a subscript of $\Pr$.  For example,
$$\Pr_\pi(A_t|H_t) =  \Pr(A^\pi_t|H^\pi_t) = \pi(A^\pi_t|H^\pi_t).$$

We denote independence of random variables $X$ and $Y$ by $X \perp Y$ and conditional independence, conditioned on another random variable $Z$, by $X \perp Y | Z$. 

When expressing expectations, we use the same subscripting notation as with probabilities.  For example, the expectation of a reward $R_{t+1, A_t^\pi}$ is written as $\E[R_{t+1, A_t^\pi}] = \E_\pi[R_{t+1, A_t}]$. 

Much of the paper studies properties of interactions under a specific policy $\pi_{\rm agent}$.  When it is clear from context, we suppress superscripts and subscripts that indicate this.  For example, $H_t = H^{\pi_{\rm agent}}_t$, $A_t = A_t^{\pi_{\rm agent}}$, $R_{t+1} = R_{t+1, A_t^{\pi_{\rm agent}}}$.  Further,
$$\Pr(A_t|H_t) =  \Pr_{\pi_{\rm agent}}(A_t|H_t) = \pi_{\rm agent}(A_t|H_t) \qquad \text{and} \qquad \E[R_{t+1, A_t}] = \E_{\pi_{\rm agent}}[R_{t+1, A_t}].$$

\section{Information-Theoretic Concepts and Notation, and Some Useful Relations}
\label{sec:information}
We review some standard information-theoretic concepts and associated notation in this section. 

A central concept is the entropy $\H(X)$, which quantifies the information content or, equivalently, the uncertainty of a random variable $X$.  For a random variable $X$ that takes values in a countable set $\mathcal{X}$, we will define the entropy to be $\H(X) = -\E[\ln \Pr(X)]$, with a convention that $0 \ln 0 = 0$.  Note that we are defining entropy here using the natural rather than binary logarithm.  As such, our notion of entropy can be interpreted as the expected number of nats -- as opposed to bits -- required to identify $X$.  The realized conditional entropy $\H(X | Y = y)$ quantifies the uncertainty remaining after observing $Y=y$.  If $Y$ takes on values in a countable set $\mathcal{Y}$ then $\H(X | Y = y) = - \E[\ln \Pr(X | Y) | Y = y]$.
This can be viewed as a function $f(y)$ of $y$, and we write the random variable $f(Y)$ as $\H(X|Y = Y)$.  The conditional entropy $\H(X | Y)$ is its expectation
$\H(X | Y) = \E[\H(X|Y=Y)]$.

The mutual information $\I(X; Y) = \H(X) - \H(X|Y)$ quantifies information common to random variables $X$ and $Y$, or equivalently, the information about $Y$ required to identify $X$.  If $Z$ is a random variable taking on values in a countable set $\mathcal{Z}$ then the realized conditional mutual information $\I(X; Y| Z=z)$ quantifies remaining common information after observing $Z = z$, defined by $\I(X; Y| Z=z) = \H(X|Z=z) - \H(X|Y, Z=z)$.
The conditional mutual information $\I(X; Y | Z)$ is its expectation $\I(X; Y | Z) = \E[\I(X; Y | Z = Z)]$.

For random variables $X$ and $Y$ taking on values in (possibly uncountable) sets $\mathcal{X}$ and $\mathcal{Y}$, mutual information is defined by
$\I(X; Y) = \sup_{f \in \mathcal{F}_{\text{finite}}, g \in \mathcal{G}_{\text{finite}}} \I(f(X); g(Y))$, 
where $\mathcal{F}_{\text{finite}}$ and $\mathcal{G}_{\text{finite}}$ are the sets of functions mapping $\mathcal{X}$ and $\mathcal{Y}$ to finite ranges.  Specializing to the case where $\mathcal{X}$ and $\mathcal{Y}$ are countable recovers the previous definition.  The generalized notion of entropy is then given by $\H(X) = \I(X; X)$.  Conditional counterparts to mutual information and entropy can be defined in a manner similar to the countable case.

One representation of mutual information, which we will use, is in terms of the differential entropy.  The differential entropy $\diffentropy(X)$ of a random variable $X$ with probability density $f$ is defined by
$$\diffentropy(X) = - \int f(x) \ln f(x) dx.$$
The conditional differential entropy $\diffentropy(X|Y)$ of $X$ conditioned on $Y$ is evaluated similarly but with a conditional density function.  Finally, mutual information can be written as $\I(X; Y) = \diffentropy(X) - \diffentropy(X|Y)$.

We will also make use of total-variation distance and KL-divergence as measures of difference between distributions.
For any pair of probability measures $P$ and $P'$ defined with respect to $(\Omega, \mathcal{F})$, we denote the total-variation distance by 
\begin{align*}
    \TV(P \| P^{'}) = \sup_{A \in \mathcal{F}} | P(A) - P^{'}(A)|.  
\end{align*}
If $P$ and $P^{'}$ are discrete, then 
\begin{align*}
    \TV(P \| P^{'}) = \frac{1}{2} \sum_{\omega \in \Omega}|P(\omega) - P^{'}(\omega)|. 
\end{align*}
We denote KL-divergence by
$$\KL(P \| P') = \int P(dx) \ln\frac{dP}{dP'}(x).$$
Gibbs' inequality asserts that $\KL(P\|P') \geq 0$, with equality if and only if $P$ and $P'$ agree almost everywhere with respect to $P$.

Mutual information and KL-divergence are intimately related.  For any probability measure $P(\cdot) = \Pr((X,Y) \in \cdot)$ over a product space $\mathcal{X} \times \mathcal{Y}$ and probability measure $P'$ generated via a product of marginals $P'(dx \times dy) = P(dx) P(dy)$, mutual information can be written in terms of KL-divergence:
\begin{equation}
\label{eq:mutual-information-marginal-distribution}
\I(X; Y) = \KL(P \| P').
\end{equation}
Further, for any random variables $X$ and $Y$,
\begin{equation}
\label{eq:mutual-information-conditional-distribution}
\I(X; Y) = \E[\KL(\Pr(Y \in \cdot|X) \| \Pr(Y \in \cdot))].
\end{equation}
In other words, the mutual information between $X$ and $Y$ is the KL-divergence between the distribution of $Y$ with and without conditioning on $X$.

Pinsker's inequality provides a relation between the total-variation distance and the KL-divergence:
\begin{lemma}[Pinsker's Inequality] 
	\label{lemma:Pinsker}
	\begin{equation}\label{key}
		\TV(P \| P^{'}) \leq \sqrt{\frac{1}{2} \KL(P \| P^{'}) }. 
	\end{equation}
\end{lemma}
Mutual information satisfies the chain rule and the data-processing inequality. 
\begin{lemma} [Chain Rule for Mutual Information]
\label{lemma:chain_MI}
\begin{align*}
\I(X_1, X_2, ... ., X_n; Y) = \sum_{i = 1}^n \I(X_i; Y| X_1, X_2, ... , X_{i-1}).
\end{align*}
\end{lemma}

\begin{lemma} [Data Processing Inequality for Mutual Information]
\label{lemma:dp_MI}
If $X$ and $Z$ are independent conditioned on $Y$, then 
\begin{align*}
\I(X; Y) \geq I(X; Z). 
\end{align*}
\end{lemma}

The chain rule for KL-divergence is introduced below. 
\begin{lemma} [Chain Rule for KL-Divergence]
\label{lemma:KL_chainrule}
\begin{align*}
\KL\left(\Pr((X_1, X_2) \in \cdot) \bdbar \Pr((Y_1, Y_2) \in \cdot)\right) 
 = \KL \p{ \Pr(X_1 \in \cdot) \bdbar \Pr(Y_1 \in \cdot) } +  \E[\KL \p{\Pr(X_2 \in \cdot | X_1) \bdbar \Pr(Y_2 \in \cdot | Y_1 \leftarrow X_1) }]. 
\end{align*}
\end{lemma}
As a direct consequence of the chain rule for KL-divergence, we have the following corollary.

\begin{corollary} 
	\label{cor:KL_chainrule}
\begin{enumerate} [I.]
\item 
\begin{equation}\label{key}
	   \KL \p{\Pr(X_1 \in \cdot) \bdbar \Pr(Y_1 \in \cdot) } \leq
	\KL \p{\Pr((X_1, X_2) \in \cdot) \bdbar \Pr((Y_1, Y_2) \in \cdot)} . 
\end{equation}\item 
\begin{equation}\label{key}
	    \KL \p {\Pr((X_1, ... , X_n) \in \cdot) \bdbar \Pr((Y_1, ... , Y_n) \in \cdot) }
	= \sum_{k = 1}^n  \E\left[\KL \p{ \Pr(X_k \in \cdot | X_{1 : k - 1}) \bdbar \Pr(Y_k \in \cdot | Y_{1 : k - 1} \leftarrow X_{1 : k-1}) }\right]. 
\end{equation}
\end{enumerate}
\end{corollary}

We introduce the data-processing inequality for KL-divergence.
\begin{lemma}[Data-Processing Inequality for KL-Divergence]
	 \label{lemma:KL_dataprocessing}
\noindent 
Suppose that conditional probability distribution of $X_2$ given $X_1$ is the same as the conditional probability distribution of $Y_2$ given $Y_1$.

Then, 
\begin{align*}
    \KL \p{\Pr(X_2 \in \cdot) \bdbar \Pr(Y_2 \in \cdot)} \leq \KL \p{\Pr(X_1 \in \cdot) \bdbar \Pr(Y_1 \in \cdot)}. 
\end{align*}
\end{lemma}

Below we provide a proof to Lemma~\ref{lemma:expecatation_tv}, which bounds the difference between expectations by the total-variation distance. We restate the lemma below. 
\lemmaexpectationtv*
\begin{proof} We have
    \begin{align*}
        |\E[X] - \E[Y]| 
        = &\ \left|\sum_{x \in \mathcal{X}} x \left(\Pr(X = x) - \Pr(Y = x)\right)\right|\\
        = &\ \left|\sum_{x \in \mathcal{X}} \left(x - \frac{1}{2}B\right) \left(\Pr(X = x) - \Pr(Y = x)\right)\right|\\
        \leq &\ \frac{B}{2} \sum_{x \in \mathcal{X}} |(\Pr(X = x) - \Pr(Y = x)| \\
        = &\ B \TV(\Pr(X \in \cdot) \| \Pr(Y \in \cdot)). 
    \end{align*}
\end{proof}

\section{Multivariate Gaussian}

\subsection{Posterior Updates}
\label{app:pos_updates}
\label{app:Gaussian_Prop}
The following  lemmas give expressions for various relevant posterior distributions. 
\begin{restatable}{lemmma}{lemmabayes} 
	\label{lemma:bayes}
	For all $t \in \mathbb{Z}_+$, let
	$\Lambda_t = \sum_{i = 0}^{t-1} \1_{A_i} \1_{A_i}^{\top}$. 
	Then, conditional on observing $H_t$, the posterior of $\tilde{\theta}$ is Gaussian and
	distributed according to $\mathcal{N}(\mu_t, \Sigma_t)$, with
	\begin{align*}
		\mu_t = &\ \left(\Sigma_0^{-1} + \frac{1}{{\sigma}^2} \Lambda_t\right)^{-1} \left(\Sigma_0^{-1} \mu_0 + \frac{1}{{\sigma}^2} \sum_{i = 0}^{t-1}\1_{A_i} R_{i + 1, A_i}\right),\\
		\Sigma_t = &\ \left(\Sigma_0^{-1} + \frac{1}{{\sigma}^2} \Lambda_t\right)^{-1}.
	\end{align*}
\end{restatable}
\begin{proof}
Recall that $\tilde{\theta} \sim \mathcal{N}(\mu_0, \Sigma_0)$, and the noise variance is $\sigma^2$. By Bayes rule, 
\begin{align*}
    p(\theta | H_t) \propto &\ p(\theta) p(H_t | \theta) \\
    \propto &\ \exp\left(-\frac{1}{2}(\theta - \mu_0)^{\top} \Sigma_0^{-1}(\theta - \mu_0)\right)
    \prod_{i = 0}^{t-1} \exp\left(-\frac{(R_{i+1, A_{i}} - \theta_{A_i})^2}{2\sigma^2}\right) \\
    \propto &\ \exp\left(-\frac{1}{2}(\theta - \mu_0)^{\top} \Sigma_0^{-1}(\theta - \mu_0)\right)
     \exp\left(-\frac{1}{2 \sigma^2}\sum_{i = 0}^{t-1}(R_{i + 1} - \theta)^{\top}\1_{A_i} 1_{A_i}^{\top}(R_{i + 1} - \theta)\right) \\
    \propto &\ \exp\left(-\frac{1}{2}\left[\theta^{\top} \Sigma_0^{-1}\theta + \frac{1}{\sigma^2} \sum_{i = 0}^{t-1} \theta^{\top}\1_{A_i}\1_{A_i}^{\top}\theta
    -2 \mu_0^{\top}\Sigma_0^{-1}\theta - \frac{2}{\sigma^2} \sum_{i = 0}^{t-1}R_{i + 1}^{\top}\1_{A_i}\1_{A_i}^{\top}\theta\right]\right) \\
    \propto &\ \exp\left(-\frac{1}{2}\left[\theta^{\top} \left(\Sigma_0^{-1} + \frac{1}{\sigma^2}\sum_{i = 0}^{t-1}\1_{A_i}\1_{A_i}^{\top}\right)\theta 
    -2\left(\mu_0^{\top}\Sigma_0^{-1} + \frac{1}{\sigma^2} \sum_{i = 0}^{t-1}R_{i + 1}^{\top}\1_{A_i}\1_{A_i}^{\top}\right) \theta \right]\right) \\
    \propto &\ \exp\left(-\frac{1}{2}(\theta - \mu_t)^{\top} \Sigma_t^{-1}(\theta - \mu_t)\right), 
\end{align*}
where
	\begin{align*}
		\mu_t = &\ \left(\Sigma_0^{-1} + \frac{1}{{\sigma}^2} \sum_{i = 0}^{t-1}\1_{A_i}\1_{A_i}^{\top} \right)^{-1} \left(\Sigma_0^{-1} \mu_0 + \frac{1}{{\sigma}^2} \sum_{i = 0}^{t-1}\1_{A_i} \1_{A_i}^{\top} R_{i + 1}\right),\\
		\Sigma_t = &\ \left(\Sigma_0^{-1} + \frac{1}{{\sigma}^2} \sum_{i = 0}^{t-1}\1_{A_i}\1_{A_i}^{\top} \right)^{-1}.
	\end{align*}
\end{proof}

\begin{restatable}{lemmma}{lemmabayeschi} 
	\label{lemma:bayes_chi}
Under Assumption~\ref{ass:gaussian}, for all $t \in \mathbb{Z}_+$, and $h \in \mathcal{H}_t$, conditioned on $\tilde{H}_t = h$, the vector constructed by stacking $\pseudotarget$ and $\tilde{\theta}$ is distributed according to a multivariate Gaussian distribution and that the covariance matrix depends on the history only through the actions. 
\end{restatable}

\begin{proof}
Let us use $X$ to denote the vector constructed by stacking $\pseudotarget$ and $\tilde{\theta}$. By Assumption~\ref{ass:gaussian}, $X \sim \mathcal{N}(\tilde{\mu}, \tilde{\Sigma})$ for some mean $\tilde{\mu}$ and covariance $\tilde{\Sigma}$. Say $\tilde{\Sigma}$ has $d$ non-zero eigenvalues and $d_0$ zero eigenvalues. 
We can construct random variable $Y$ that takes values in $\R^{d}$ such that $Y$ is a linear transformation of $X$ and $Y \sim \mathcal{N}(\mu, \Sigma)$ with a full-rank covariance matrix $\Sigma$. Conditioned on $\tilde{H}_t = h$, $Y$ is multivariate Gaussian and the covariance matrix depends on the history only through the actions. Since $X$ is a linear transformation of $Y$, conditioned on $\tilde{H}_t = h$, $X$ is also multivariate Gaussian and the covariance matrix depends on the history only through the actions. To make things clear, we provide details on how we can construct $Y$, how $X$ can be written as a linear transformation of $Y$, and a short proof on how the posterior of $Y$ conditioned on history is multivariate Gaussian with covariance matrix independent of rewards.

First, recall that the covariance matrix of $X$ is symmetric and positive semidefinite, so it has an eigendecomposition 
\begin{align*}
\tilde{\Sigma} = Q \Lambda Q^{\top},
\end{align*}
where $Q$ is orthogonal and $\Lambda$ is diagonal.
In particular, $\Lambda$ is a diagonal matrix with $d$ non-zero entries along its diagonal and $d_0$ zero entries along its diagonal. Without loss of generality, we say the first $d$ entries along its diagonal are non-zero. We construct $Y$ as follows:
\begin{align*}
    Y = 
    \begin{bmatrix}
            I_{d \times d} & \0_{d \times d_0}
    \end{bmatrix}
    Q^{\top} X.
\end{align*}
Then $Y$ takes values in $\R^d$, and is multivariate Gaussian with a full-rank covariance matrix:
\begin{align*}
    Y \sim \mathcal{N}\left(\begin{bmatrix}
            I_{d \times d} & \0_{d \times d_0}
    \end{bmatrix}
    Q^{\top} \mu, 
    \begin{bmatrix}
            I_{d \times d} & \0_{d \times d_0}
    \end{bmatrix}
    \Lambda
    \begin{bmatrix}
        I_{d \times d} \\
        \0_{d_0 \times d}
    \end{bmatrix}
    \right).
\end{align*}

\noindent We let $b \in \mathbb{R}^{d + d_0}$ be
\begin{align*}
    b = 
        \begin{bmatrix}
        \0_{d \times d} & \0_{d \times d_0}\\
        \0_{d_0 \times d} & I_{d_0 \times d_0}
    \end{bmatrix}
    Q^{\top} \mu
    = 
    \begin{bmatrix}
        \0_{d \times d} & \0_{d \times d_0}\\
        \0_{d_0 \times d} & I_{d_0 \times d_0}
    \end{bmatrix}
    Q^{\top} X.
\end{align*}
Then $X$ can be written as a linear transformation of $Y$ as follows:
\begin{align*}
    X = Q\left(\begin{bmatrix}
        I_{d \times d} \\
        \0_{d_0 \times d}
    \end{bmatrix} Y + b\right).
\end{align*}
Now we show that for all $t \in \mathbb{Z}_+$ and $h \in \mathcal{H}_t$, conditioned on $\tilde{H}_t = h$, $Y$ is multivariate Gaussian and its covariance does not depend on rewards. Recall that $Y \sim \mathcal{N}(\mu, \Sigma)$. 
 By Bayes rule, for all $t \in \mathbb{Z}_+$ and $h \in \mathcal{H}_t$, we have
\begin{align*}
    p\left(Y | \tilde{H}_t = h\right) \propto &\ p(Y) p\left(\tilde{H}_t = h | Y\right) \\
    \propto &\ \exp\left(-\frac{1}{2}\left(Y - \mu\right)^{\top} \Sigma^{-1}\left(Y - \mu\right)\right)
    \prod_{i = 0}^{t-1} \exp\left(-\frac{(R_{i+1, A_{i}} - \theta_{A_i})^2}{2\sigma^2}\right) \\
        \propto &\ \exp\left(-\frac{1}{2}\left(Y - \mu\right)^{\top} \Sigma^{-1}(Y - \mu)\right)
    \prod_{i = 0}^{t-1} \exp\left(-\frac{(R_{i+1}^{\top} \1_{A_i} - X^{\top} \1_{A_i})^2}{2\sigma^2}\right) \\
            \propto &\ \exp\left(-\frac{1}{2}\left(Y - \mu\right)^{\top} \Sigma^{-1}
            \left(Y - \mu\right)\right)
    \prod_{i = 0}^{t-1} \exp\left(-\frac{\left\{R_{i+1}^{\top} \1_{A_i} - \left(Y^{\top} \begin{bmatrix}
            I_{d \times d} & \0_{d \times d_0}
    \end{bmatrix}
     + b^{\top}\right) Q^{\top} \1_{A_i}\right\}^2}{2\sigma^2}\right) \\
    \propto &\ \exp\left(-\frac{1}{2}\left(Y - \mu\right)^{\top} \Sigma^{-1}\left(Y - \mu\right)\right)
     \exp\left(-\frac{1}{2 \sigma^2}
     \sum_{i = 0}^{t - 1}\left\{\left[-Y^{\top} 
     \begin{bmatrix}
             I_{d \times d} & \0_{d \times d_0}
              \end{bmatrix}
              Q^{\top} + \left(R_{i + 1}^{\top} - b^{\top} Q^{\top}\right)
    \right] \1_{A_i}\right\}^2
     \right) \\
    \propto &\ \exp\left(-\frac{1}{2}\left[Y^{\top} \left(\Sigma^{-1} + \frac{1}{\sigma^2}
    \begin{bmatrix}
             I_{d \times d} & \0_{d \times d_0}
        \end{bmatrix}
        Q^{\top}
    \sum_{i = 0}^{t-1}\1_{A_i}\1_{A_i}^{\top}
    Q
    \begin{bmatrix}
             I_{d \times d} \\ \0_{d_0 \times d}
        \end{bmatrix}
    \right)Y 
    \right]\right)  \\
     &\ \exp \left(-\frac{1}{2}\left[-2    \left(\mu^{\top}\Sigma^{-1} + \frac{1}{\sigma^2} 
    \sum_{i = 0}^{t-1}\left(R_{i + 1}^{\top} - b^{\top} Q^{\top}\right) \1_{A_i}\1_{A_i}^{\top}
    Q
    \begin{bmatrix}
             I_{d \times d} \\ \0_{d_0 \times d} 
    \end{bmatrix}
    \right) Y \right]\right) \\
    \propto &\ \exp\left(-\frac{1}{2}(Y - \mu_h)^{\top}\Sigma_h^{-1}(Y - \mu_h) \right),
\end{align*}
where 
\begin{align*}
    \mu_h =&\ \Sigma_h
        \left(\Sigma^{-1} \mu + \frac{1}{\sigma^2}
        \begin{bmatrix}
             I_{d \times d} & \0_{d \times d_0}
        \end{bmatrix}
        Q^{\top} \sum_{i = 0}^{t-1} \1_{A_i}\1_{A_i}^{\top} (R_{i + 1} - Qb)
        \right),\\
    \Sigma_h = &\ \left(\Sigma^{-1} + \frac{1}{\sigma^2}
    \begin{bmatrix}
             I_{d \times d} & \0_{d \times d_0}
        \end{bmatrix}
        Q^{\top}
    \sum_{i = 0}^{t-1}\1_{A_i}\1_{A_i}^{\top}
    Q
    \begin{bmatrix}
             I_{d \times d} \\ \0_{d_0 \times d}
        \end{bmatrix}
        \right)^{-1}.
\end{align*}
The posterior distribution is clearly multivariate Gaussian, and its covariance matrix $\Sigma_h$ does not depend on the rewards.
\end{proof}

\subsection{Entropy and Mutual Information}
\label{app:Gaussian_info}
\begin{lemma}
	 \label{lemma:gaussian_entropy}
	Fix $\mu \in \R^n$, $\Sigma \in \mathcal{S}^n_{++}$, let $X \sim \mathcal{N}(\mu, \Sigma)$. Then the differential entropy of $X$ is 
	\begin{align*}
		\diffentropy(X) = \frac{1}{2} \ln |\Sigma| + \frac{n}{2} \ln (2 \pi e). 
	\end{align*}
	When $n = 1$, $\diffentropy(X) = \frac{1}{2} \ln (2 \pi e \sigma^2)$. 
\end{lemma}
\noindent A proof of the case when $n = 1$ can be found in Example 8.1.2 of \citep{CoverThomas2006}.

\gaussianmutualinformation*
\begin{proof}
The differential entropy of $\tilde{\theta}$, or for that matter, of any Gaussian random variable with covariance matrix $\Sigma_0$ is
$$\diffentropy\left(\tilde{\theta}\right) = \frac{1}{2} \ln\left((2\pi e)^\actions |\Sigma_0|\right),$$
by Lemma~\ref{lemma:gaussian_entropy}. 
Based on this expression, we have
\begin{align*}
\I\left(\hat{\theta}; \pseudoenvironment\right)
= &\ \I\left(\hat{\theta}; \tilde{\theta}\right) \\
=&\ \diffentropy\left(\tilde{\theta}\right) - \diffentropy\left(\tilde{\theta} | \hat{\theta}\right) \\
=&\ \diffentropy\left(\tilde{\theta}\right) - \diffentropy(Z) \\
=&\ \frac{1}{2} \ln\left((2\pi e)^\actions |\Sigma_0|\right) - \frac{1}{2} \ln\left((2\pi e)^\actions |\delta^2 \Sigma_0|\right) \\
=&\ \frac{1}{2} \ln\left(\frac{|\Sigma_0|}{\delta^{2 \actions} |\Sigma_0|}\right) \\
=&\ \frac{\actions}{2} \ln\left(\frac{1}{\delta^2}\right).
\end{align*}
\end{proof}

\section{A General Regret Bound: Proof of Theorem~\ref{th:general-regret-bound}}
\label{sec:general_regret_bound}
\generalregretbound*

\begin{proof}
We upper-bound the regret:
\begin{align*}
\regret(T)
=& \sum_{t=0}^{T-1} \E[R_* - R_{t+1, A_t}] \\
\leq& \sum_{t=0}^{T-1} \E[\E[R_* - R_{t+1,A_t} - \epsilon | H_t]_+] + \epsilon T \\
\overset{(a)}{\leq}& \sum_{t=0}^{T-1} \E\left[\sqrt{\Gamma_{\target, \epsilon} \I(\target; A_t, R_{t+1, A_t}|H_t=H_t)}\right] + \epsilon T \\
\overset{(b)}{\leq}& \sum_{t=0}^{T-1} \sqrt{\Gamma_{\target, \epsilon} \E[\I(\target; A_t, R_{t+1, A_t}|H_t=H_t)]} + \epsilon T \\
=& \sum_{t=0}^{T-1} \sqrt{\Gamma_{\target, \epsilon} \I(\target; A_t, R_{t+1, A_t}|H_t)} + \epsilon T \\
\overset{(c)}{\leq}& \sqrt{\sum_{t=0}^{T-1} \I(\target; A_t, R_{t+1, A_t}|H_t)} \sqrt{\Gamma_{\target, \epsilon} T} + \epsilon T,\numberthis
\label{eq:grb_eq_2}
\end{align*}
where step $(a)$ follows from the definition of the information ratio, step $(b)$ follows from Jensen's inequality, and step $(c)$ follows from the Cauchy-Bunyakovsky-Schwarz inequality.

By the chain rule of mutual information (Lemma~\ref{lemma:chain_MI} in Appendix~\ref{sec:information}) and the data processing inequality of mutual information (Lemma~\ref{lemma:dp_MI} in Appendix~\ref{sec:information}) and that $\target$ and $H_{\infty}$ are independent conditioned on $\environment$, we have
\begin{align*}\sum_{t=0}^{T-1} \I(\target; A_t, R_{t+1, A_t}|H_t) 
= \I(\target; H_T) 
\leq \I(\target; \environment). \numberthis
\label{eq:grb_eq_1}
\end{align*}
We establish the desired bound by plugging \eqref{eq:grb_eq_1} into \eqref{eq:grb_eq_2}. 
\end{proof}

\section{Construction of the Auxiliary Rewards: Proof of Lemma~\ref{lemma:rdagger}}
\label{app:lemma:rdagger}
\lemmardaggerproperties*
\begin{proof}
We define $R^{\dagger}$ as a function of $\tilde{\theta}$ and $\tilde{R}$ as follows. For all $t \in \mathbb{Z}_+$ and $a \in \mathcal{A}$: 
\begin{enumerate}
    \item If $\tilde{R}_{t, a} \in \{0, 1\}$, then $R^\dagger_{t,a}= \tilde{R}_{t,a}$.
    \item If $\tilde{R}_{t, a} \notin \{0, 1\}$ and $\tilde{\theta}_a \notin [0, 1]$, then $R^\dagger_{t,a}= 2 \gamma$.
    \item Otherwise, let $R^{\dagger}_{t,a}$ be drawn i.i.d.~from a Bernoulli distribution with parameter $\tilde{\theta}_a$, independently from the rest of the system. 
\end{enumerate} 
Next, we construct the auxiliary action and history, $A^\dagger_t$ and $H^{\dagger}_t$, in a recursive manner, as follows. 
\begin{enumerate}
	\item Let $H_0^{\dagger} = H_0$ be the empty history. 
	\item For all $t \geq 1$, we sample $A^\dagger_t$ from $\pi\p{\cdot | H^{\dagger}_t}$ and define $H^{\dagger}_{t + 1} = \p{H^{\dagger}_{t},  A^{\dagger}_{t}, R^{\dagger}_{t + 1, A^{\dagger}_{t}}}$. Here, the randomness used in sampling $A^{\dagger}_t$ for each $t \geq 1$ is independent of the rest of the system. 
\end{enumerate}

We now demonstrate that the above construction possesses the desirable properties. First, by construction, properties $(i)$, $(ii)$, and $(iii)$ are automatically satisfied. Now we prove property $(iv)$. 
It suffices to show that for all $t \in \mathbb{Z}_+$ and $h_t = (a_0, r_{1, a_0}, ... , a_{t-1}, r_{t, a_{t-1}}) \in \mathcal{H}_t$, we have
\begin{align} \label{eq:dagger}
       \E\left[\tilde{R}_{t + 1, \tilde{A}_t} \bbar  \tilde{H}_t = h_t\right] \leq \E\left[{R}^{\dagger}_{t + 1, {A}^{\dagger}_t} \bbar  {H}^{\dagger}_t = h_t\right].
\end{align}
We have for all $t \in \mathbb{Z}_+$ and $h_t \in \mathcal{H}_t$:

\begin{align} \label{eq:dagger_result}
  \E\left[\tilde{R}_{t+1,\tilde{A}_t} \bbar   \tilde{H}_t = h_t\right] \nonumber
=  &\ \E\left[\tilde{\theta}_{\tilde{A}_t} \bbar   \tilde{H}_t = h_t\right]\\ \nonumber
\stackrel{(a)}{=} &\ \sum_{a \in \mathcal{A}} \pi(a | h_t) \E\left[\tilde{\theta}_{ a} \bbar  \tilde{H}_t = h_t\right]\\
\stackrel{}{=} &\ \sum_{a \in \mathcal{A}} \pi(a | h_t) \left( \E\left[\tilde{\theta}_{a} \1_{\left\{\tilde{\theta}_a \in [0, 1]\right\}} \bbar  \tilde{H}_t = h_t\right] 
+ \E\left[\tilde{\theta}_{a} \1_{\left\{\tilde{\theta}_a \notin [0, 1]\right\}} \bbar  \tilde{H}_t  = h_t\right]\right),
\end{align}
where step $ (a) $ follows from the definition of $\tilde{A}_t$. 

Before we proceed to show (\ref{eq:dagger}), we introduce the following lemmas.

\begin{restatable}{lemmma}{lemmasymmetry}
\label{lemma:symmetry}
Let $X$ be a random variable with a  distribution that is symmetric around $\mu \in \R$.
Then 
\begin{align*}
    \E\left[X \1_{\{X \notin [0, 1]\}}\right] \leq 2\mu_{+} \Pr\left(X \notin [0, 1]\right),
\end{align*}
where $\mu_{+} = \max\{\mu, 0\}$. 
\end{restatable}
\begin{proof} We prove the statement in two cases:

	\emph{Case 1.} If $\mu < \frac{1}{2}$, then
	\begin{align*}
		\E\left[X \1_{\{X \notin [0, 1]\}}\right]
		= &\ \E\left[X \1_{\{X < 2 \mu - 1\}}\right] + \E\left[X \1_{\{2\mu - 1 \leq X < 0\}}\right] + \E\left[X \1_{\{X > 1\}}\right] \\
		\stackrel{(a)}{=} &\ 2 \mu \Pr\left(X > 1\right) + \E\left[X \1_{\{2\mu -1 \leq X < 0\}}\right] \\
		< &\ 2 \mu \Pr\left(X > 1\right) \\
		\leq &\ 2 \mu_{+} \Pr\left(X \notin [0, 1]\right),
	\end{align*}
	where $(a)$ follows from the fact that $2\mu - 1$ and $1$ are symmetric around $\mu$ and that the distribution of $X$ is symmetric around $\mu$. \\
	
\emph{Case 2.} If $\mu \geq \frac{1}{2}$, then 
	\begin{align*}
		\E\left[X \1_{\{X \notin [0, 1]\}}\right] 
		= &\ \E\left[X \1_{\{X < 0\}}\right] + \E\left[X \1_{\{1 < X \leq 2\mu\}}\right] + \E\left[X \1_{\{X > 2 \mu\}}\right]\\
		\stackrel{(a)}{=} &\ 2 \mu \Pr\left(X > 2\mu\right) + \E\left[X \1_{\{1 < X \leq 2\mu\}}\right] \\
		\leq &\ 2 \mu \Pr\left(X > 2\mu\right) + 2\mu \Pr\left(1 < X \leq 2\mu\right) \\
		= &\ 2 \mu \Pr\left(X > 1\right) \\
		\leq &\ 2 \mu_{+} \Pr
		\left(X \notin [0, 1]\right), 
	\end{align*}
	where $(a)$ follows from the fact that $0$ and $2 \mu$ are symmetric around $\mu$ and that the distribution of $X$ is symmetric around $\mu$. 
\end{proof}

\begin{restatable}{lemmma}{lemmacoupling}
\label{lemma:coupling}
For all $t \in \mathbb{Z}_+$, $h_t \in \mathcal{H}_t$, and $a \in \mathcal{A}$, 
\begin{align*}
\Pr\left(\tilde{\theta}_a \in \cdot \bbar \tilde{H}_t = h_t\right) = \Pr\left(\tilde{\theta}_a \in \cdot \bbar H^{\dagger}_t = h_t\right).
\end{align*}
\end{restatable}

\begin{proof}
	First, we can rewrite the expression using $h_t = (a_0, r_{1, a_0}, ... , a_{t-1}, r_{t, a_{t-1}})$ as follows: 
	\begin{align*}
		\Pr\left(\tilde{\theta}_a \in \cdot \bbar \tilde{H}_t = h_t\right) 
		= &\  \Pr\left(\tilde{\theta}_a \in \cdot \bbar  \tilde{A}_0 = a_0, \tilde{R}_{1, \tilde{A}_0} = r_{1, a_0},\ ...\ ,\ \tilde{A}_{t-1} = a_{t-1}, \tilde{R}_{t, \tilde{A}_{t-1}} = r_{t, a_{t-1}}\right)\\
		= &\  \Pr\left(\tilde{\theta}_a \in \cdot \bbar   \tilde{R}_{1, a_0} = r_{1, a_0},\ ...\ ,\ \tilde{R}_{t, a_{t-1}} = r_{t, a_{t-1}}\right).
	\end{align*}
	Observe that $r_{t+1, a_{t}} \in \{0, 1\}$ for all $t \geq 1$, $h_t \in \mathcal{H}_t$. So
 step 1 of the construction of $R^{\dagger}$ ensures that for all $t \in \mathbb{Z}_+$ and for all $h_t \in \mathcal{H}_t$, $\tilde{R}_{t+1,a_{t}} = r_{t+1, a_t}$ if and only if $R^{\dagger}_{t+1,a_{t}} = r_{t+1, a_t}$. Hence, for all $t \in \mathbb{Z}_+$ and $h_t \in \mathcal{H}_t$, it follows that 
	\begin{align*}
		\Pr\left(\tilde{\theta}_a \in \cdot \bbar \tilde{H}_t = h_t\right)  
		= &\  \Pr\left(\tilde{\theta}_a \in \cdot \bbar   {R}^{\dagger}_{1, a_0} = r_{1, a_0},\ ...\ ,\ {R}^{\dagger}_{t, a_{t-1}} = r_{t, a_{t-1}}\right) \\
		= &\ \Pr\left(\tilde{\theta}_a \in \cdot \bbar  A^{\dagger}_0 = a_0, {R}^{\dagger}_{1, A^{\dagger}_0} = r_{1, a_0},\ ...\ ,\ A^{\dagger}_{t-1} = a_{t-1}, R^{\dagger}_{t, A^{\dagger}_{t-1}} = r_{t, a_{t-1}}\right) \\
		= &\ \Pr\left(\tilde{\theta}_a \in \cdot \bbar H^{\dagger}_t = h_t\right).
	\end{align*}
\end{proof}

By Lemma \ref{lemma:bayes} in Appendix \ref{app:Gaussian_Prop}, we know that 
for all $t \in \mathbb{Z}_+$ and $h_t \in \mathcal{H}_t$,  the posterior distribution of $\tilde{\theta}_a$ conditional on $\tilde{H}_t = h_t$ is Gaussian, and is therefore symmetric around its mean. Furthermore, we have $\gamma \geq 1$ and 
\begin{equation}\label{key}
	\mathbb{E}\left[\tilde{\theta}_a \bbar  \tilde{H}_t = h_t\right] \leq \gamma. 
\end{equation}
These two facts, along with Lemma \ref{lemma:symmetry}, imply that 
\begin{align} \label{eq:gamma_result}
 \E\left[\tilde{\theta}_{a} \1_{\left\{\tilde{\theta}_a \notin [0, 1]\right\}} \bbar  \tilde{H}_t  = h_t\right]
 \leq &\ 2 \E\left[\tilde{\theta}_a | \tilde{H}_t = h_t\right]_{+} \Pr\left(\tilde{\theta}_a \notin [0, 1] \bbar  \tilde{H}_t = h_t\right)\\
 \leq &\ 2 \gamma \Pr\left(\tilde{\theta}_a \notin [0, 1] \bbar  \tilde{H}_t = h_t\right). 
\end{align}
Combining (\ref{eq:dagger_result}) and (\ref{eq:gamma_result}), and applying Lemma \ref{lemma:coupling}, we have 
\begin{align*}
 \E\left[\tilde{R}_{t+1,\tilde{A}_t} \bbar   \tilde{H}_t = h_t\right] 
 \stackrel{}{\leq} &\ \sum_{a \in \mathcal{A}} \pi(a | h_t) \left( \E\left[\tilde{\theta}_{a} \1_{\left\{\tilde{\theta}_a \in [0, 1]\right\}} \bbar  \tilde{H}_t = h_t\right]+ 2 \gamma \Pr\left(\tilde{\theta}_a \notin [0, 1] \bbar  \tilde{H}_t = h_t\right)\right). \\
  \stackrel{(a)}{=} &\ \sum_{a \in \mathcal{A}} \pi(a | h_t) \left( \E\left[\tilde{\theta}_{a} \1_{\left\{\tilde{\theta}_a \in [0, 1]\right\}} \bbar  {H}^{\dagger}_t = h_t\right]+ 2 \gamma \Pr\left(\tilde{\theta}_a \notin [0, 1] \bbar  {H}^{\dagger}_t = h_t\right)\right). \\
\stackrel{}{=} &\ \sum_{a \in \mathcal{A}} \pi(a | h_t) \left( \E\left[R^{\dagger}_{t+1,a} \1_{\left\{\tilde{\theta}_a \in [0, 1]\right\}} \bbar  H^{\dagger}_t = h_t\right]
 +  \E\left[R^{\dagger}_{t+1,a} \1_{\left\{\tilde{\theta}_a \notin [0, 1]\right\}} \bbar  H^{\dagger}_t = h_t\right] \right)\\
 \stackrel{}{=} &\ \sum_{a \in \mathcal{A}} \pi(a | h_t) \E\left[R^{\dagger}_{t,a} \bbar  H^{\dagger}_t = h_t\right]\\
\stackrel{(b)}{=} &\ \E\left[R^{\dagger}_{t+1, {A}^{\dagger}_t} \bbar H^{\dagger}_t = h_t\right],
\end{align*}
where step $ (a) $ follows from Lemma \ref{lemma:coupling} and $ (b)  $  from the definition of $A^{\dagger}_t$.  This proves  \eqref{eq:dagger}, and consequently, \eqref{eq:dagger_exp}. 
\end{proof}

\section{Bounding $\gamma$: Proof of Lemma \ref{lemma:diagonally_dominant}}
In this section, we prove Lemma~\ref{lemma:diagonally_dominant}, which establishes an upper bound on $\gamma$. Fix a $n \times n$ matrix $A$, we define 
\begin{align*}
    \alpha(A) = \min_{1 \leq i \leq n} \left(|A_{ii}| - \sum_{j \neq i} |A_{ij}|\right).
\end{align*}
Recall that we say that matrix $A$ is diagonally dominant if $\alpha(A) \geq 0$, and we say that $A$ is strictly diagonally dominant if the inequality is strict.

Below we restate Lemma~\ref{lemma:diagonally_dominant} before proving it.

\label{sec:gamma_bound}
\boundedgamma*
\begin{proof}
For all $t \in \mathbb{Z}_+$, and $h = (a_0, r_{1, a_0}, ... , a_{t-1}, r_{t, a_{t-1}}) \in \mathcal{H}_t$, 
let
	$\Lambda_t = \sum_{i = 0}^{t-1} \1_{a_i} \1_{a_i}^{\top}$, and let $\overline{r}_t$ be a vector in $\R^{\mathcal{A}}$ where its $a$-th element is defined as: \begin{align*}
     \overline{r}_{t,a} = \sum_{i = 0}^{t-1} r_{i + 1, a_i} \1_{\left\{a_i = a\right\}}  \bigg/ \max\left( \sum_{i = 0}^{t-1}\1_{\left\{a_i = a\right\}}, 1\right).
 \end{align*}
By Lemma \ref{lemma:bayes} in Appendix \ref{app:pos_updates}, we have for all $t \in \mathbb{Z}_+$, and $h \in \mathcal{H}_t$, 
\begin{align*}
      \mu_t \stackrel{\triangle}{=} \E\left[\tilde{\theta}| \tilde{H}_t = h\right] = &\ \left(\Sigma_0^{-1} + \frac{1}{{\sigma}^2} \Lambda_t\right)^{-1} \left(\Sigma_0^{-1} \mu_0 + \frac{1}{{\sigma}^2} \Lambda_t \overline{r}_t \right)\\
      = &\ \left(\Sigma_0^{-1} + \frac{1}{{\sigma}^2} \Lambda_t\right)^{-1} \left(\Sigma_0^{-1} \mu_0 + \frac{1}{{\sigma}^2}\Lambda_t \mu_0 - \frac{1}{{\sigma}^2}\Lambda_t \mu_0 + \frac{1}{{\sigma}^2} \Lambda_t \overline{r}_t \right) \\
      = &\ \mu_0 + \left(\Sigma_0^{-1} + \frac{1}{{\sigma}^2} \Lambda_t\right)^{-1} \frac{1}{{\sigma}^2} \Lambda_t \left(\overline{r}_t - \mu_0 \right).
\end{align*}
We have $\alpha(\Sigma_0^{-1}) > 0$ since $\Sigma_0^{-1}$ is strictly diagonally dominant. 
Then there exists $K \in \mathbb{N}$ such that for all $k \geq K$, we have $\frac{1}{k} < \alpha(\Sigma_0^{-1})$. Then for all $t \in \mathbb{Z}_+$ and $k \geq K$, we can re-write $\mu_t$ as follows:
\begin{align*}
      \mu_t = &\ \mu_0 + \left(\Sigma_0^{-1} + \frac{1}{{\sigma}^2} \Lambda_t\right)^{-1} \left(\frac{1}{{\sigma}^2} \Lambda_t + \frac{1}{k}I - \frac{1}{k}I \right) \left(\overline{r}_t - \mu_0 \right) \\
      = &\ \mu_0 + \left\{\left(\Sigma_0^{-1} + \frac{1}{{\sigma}^2} \Lambda_t\right)^{-1} \left(\frac{1}{{\sigma}^2} \Lambda_t + \frac{1}{k} I\right) - \frac{1}{k} \left(\Sigma_0^{-1} + \frac{1}{{\sigma}^2} \Lambda_t\right)^{-1} \right\}\left(\overline{r}_t - \mu_0 \right) \\
      \stackrel{(a)}{=} &\ \mu_0 + \left\{\left[\left(\frac{1}{{\sigma}^2} \Lambda_t + \frac{1}{k} I\right)^{-1}\left(\Sigma_0^{-1} + \frac{1}{{\sigma}^2} \Lambda_t\right)\right]^{-1}  - \frac{1}{k} \left(\Sigma_0^{-1} + \frac{1}{{\sigma}^2} \Lambda_t\right)^{-1} \right\}\left(\overline{r}_t - \mu_0 \right) \\
      = &\ \mu_0 + \left\{\left[\left(\frac{1}{{\sigma}^2} \Lambda_t + \frac{1}{k} I\right)^{-1}\left(\Sigma_0^{-1} - \frac{1}{k} I + \frac{1}{{\sigma}^2} \Lambda_t + \frac{1}{k} I \right)\right]^{-1} - \frac{1}{k} \left(\Sigma_0^{-1} + \frac{1}{{\sigma}^2} \Lambda_t\right)^{-1} \right\}\left(\overline{r}_t - \mu_0 \right) \\
      = &\ \mu_0 + \left\{\left[I + \left(\frac{1}{{\sigma}^2} \Lambda_t + \frac{1}{k} I\right)^{-1}\left(\Sigma_0^{-1} - \frac{1}{k} I \right)\right]^{-1} - \frac{1}{k} \left(\Sigma_0^{-1} + \frac{1}{{\sigma}^2} \Lambda_t\right)^{-1} \right\}\left(\overline{r}_t - \mu_0 \right), \numberthis
    \label{eq:diagonally_dominant_0}
\end{align*}
where we have $(a)$ since $\frac{1}{\sigma^2} \Lambda_t + \frac{1}{k} I$ is a diagonal matrix with positive entries along its diagonal and is thus revertible. 

Recall that $\frac{1}{k} < \alpha(\Sigma_0^{-1})$, so $\Sigma_0^{-1}-\frac{1}{k} I$ is strictly diagonally dominant. In addition, observe that $\left(\frac{1}{{\sigma}^2} \Lambda_t + \frac{1}{k} I\right)^{-1}$
is a diagonal matrix with positive entries along its diagonal. So $\left(\frac{1}{{\sigma}^2} \Lambda_t + \frac{1}{k} I\right)^{-1}\left(\Sigma_0^{-1} -  \frac{1}{k} I \right)$ is strictly diagonally dominant. Hence, 
\begin{align}
\alpha\left(I + \left(\frac{1}{{\sigma}^2} \Lambda_t +  \frac{1}{k} I\right)^{-1}\left(\Sigma_0^{-1} -  \frac{1}{k} I \right)\right) > 1. 
\label{eq:diagonally_dominant_1}
\end{align}
In addition, since $\frac{1}{\sigma^2}\Lambda_t$ is a diagonal matrix with non-negative entries along its diagonal, we have 
\begin{align}
    \alpha\left(\Sigma_0^{-1} + \frac{1}{\sigma^2}\Lambda_t\right) \geq \alpha(\Sigma_0^{-1}). 
\label{eq:diagonally_dominant_2}
\end{align}

We introduce the following result established in \citep{VARAH19753} that provides an upper bound on the infinity norm of the inverse of a diagonally dominant matrix. Recall that for a $n \times n$ matrix $A$, the infinity norm of $A$ is defined as 
$\|A\|_{\infty} = \max_{1 \leq i \leq n} \sum_{j = 1}^n |A_{ij}|$. 
\label{app:bounded_posterior_mean}
\begin{lemma} \label{lemma:diag_dom}
Assume $A$ is a $n \times n$ diagonally dominant matrix. Then 
\begin{align*}
    \Vert A^{-1}\Vert_{\infty} < 1/\alpha(A), 
\end{align*}
\end{lemma}
By Lemma \ref{lemma:diag_dom}, \eqref{eq:diagonally_dominant_1} and \eqref{eq:diagonally_dominant_2} imply that
\begin{align}
\label{eq:inverse_bound}
   \left\Vert \left[I + \left(\frac{1}{{\sigma}^2} \Lambda_t +  \frac{1}{k} I\right)^{-1}\left(\Sigma_0^{-1} -  \frac{1}{k} I \right)\right]^{-1}\right\Vert_{\infty} < 1,\ 
   \left\Vert \left(\Sigma_0^{-1} + \frac{1}{{\sigma}^2} \Lambda_t \right)^{-1}\right\Vert_{\infty} < \frac{1}{\alpha(\Sigma_0^{-1})}. 
\end{align}
Recall that for all $t \in \mathbb{Z}_+$, and $h \in \mathcal{H}_t$, we have $\overline{r}_t \in [0, 1]^{\actions}$. Then it follows from \eqref{eq:diagonally_dominant_0} and \eqref{eq:inverse_bound} that 
\begin{align*}
      \mu_t = &\ \mu_0 + \left\{\left[I + \left(\frac{1}{{\sigma}^2} \Lambda_t + \frac{1}{k} I\right)^{-1}\left(\Sigma_0^{-1} - \frac{1}{k} I \right)\right]^{-1} - \frac{1}{k} \left(\Sigma_0^{-1} + \frac{1}{{\sigma}^2} \Lambda_t\right)^{-1} \right\}\left(\overline{r}_t - \mu_0 \right)\\
      \leq &\  \max_{a \in \actions} |\mu_{0, a}|\left( e + \left\Vert \left[I + \left(\frac{1}{{\sigma}^2} \Lambda_t + \frac{1}{k} I\right)^{-1}\left(\Sigma_0^{-1} - \frac{1}{k} I \right)\right]^{-1}\right\Vert_{\infty} e + \frac{1}{k} \left\Vert \left(\Sigma_0^{-1} + \frac{1}{{\sigma}^2} \Lambda_t \right)^{-1}\right\Vert_{\infty} e\right)\\
      < &\ \left(2 + \frac{1}{k \alpha(\Sigma_0^{-1})}\right) \max_{a \in \actions} |\mu_{0, a}|e,
\end{align*}
where $e$ is the $\actions$-dimensional vector with all ones. 
Let $k \rightarrow +\infty$, we conclude that for all $t \in \mathbb{Z}_+$, 
\begin{align*}
    \mu_t \leq 2 |\mu_{0, a}| e,
\end{align*}
and it follows that $\gamma \leq 2  |\mu_{0, a}|$. 
\end{proof}

\section{Examples for the Assumptions}
\subsection{An Example in Which the Optimism Assumption Holds: Proof of Lemma~\ref{lemma:gaussian_beta_dominance}}
\label{sec:ass_opt}
\gaussianbetadominance*
\begin{proof}
The proof is based on arguments developed in \citep{JMLR:v20:18-339} around stochastic optimism, the definition of which is provided in Definition~6 of the paper: A random variable $X$ is stochastically optimistic with respect to another random variable
$Y$, if for all convex increasing functions $u: \R \rightarrow \R $ $\E[u(X)] \geq \E[u(Y )]$.

Fix $t \in \mathbb{Z}_+$ and $h \in \mathcal{H}_t$. For all $a \in \actions$, let $\theta_{t,a}$ be a random variable distributed equal to the posterior distribution of $\theta_{a} | H_t = h$ and $\tilde{\theta}_{t,a}$ be a random variable distributed according to the posterior distribution of $\tilde{\theta}_{a} | \tilde{H}_t = h$.

For all $a \in \actions$, let $N_{t, a}^0 = \sum_{i = 0}^{t-1} (1 - R_{i+1, A_{i}}) \1_{\{A_i = a\}}$ and $N_{t, a}^1 = \sum_{i = 0}^{t-1} R_{i + 1, A_{i}} \1_{\{A_i = a\}}$. Then for all $a \in \actions$, $\theta_{t, a} \sim \mathrm{Beta}(\alpha_{t, a}, \beta_{t, a})$, where
\begin{align*}
    \alpha_{t, a} = &\ \alpha_a + N_{t, a}^1, \\
    \beta_{t, a} = &\ \beta_a + N_{t, a}^0. 
\end{align*}
For all $a \in \actions$, $\tilde{\theta}_{t, a} \sim \mathcal{N}(\mu_{t, a}, \sigma_{t, a}^2)$, where 
\begin{align*}
    \sigma_{t, a}^2 = &\  \frac{1}{\frac{1}{\Sigma_{0, a, a}} + \frac{N_{t, a}^0 + N_{t, a}^1}{\sigma^2}}
    \geq \frac{\sigma^2}{\alpha_{t, a} + \beta_{t, a}}, \\
    \mu_{t, a} = &\ \left(\frac{\mu_a}{\Sigma_{0, a, a}} + \frac{N_{t, a}^1}{\sigma^2}\right)\sigma_{t, a}^2
    \geq \frac{\alpha_{t, a}}{\sigma^2} \sigma_{t, a}^2 \geq \frac{\alpha_{t, a}}{\alpha_{t, a} + \beta_{t, a}}. 
\end{align*}
Recall that $\sigma^2 \geq 3$, so we can apply Lemma~4 in \citep{JMLR:v20:18-339} and conclude that for all $a \in \actions$,
\begin{align*}
    \tilde{\theta}_{t,a} \succcurlyeq_{SO} \theta_{t, a}.
\end{align*}
Since Lemma~2 in \citep{JMLR:v20:18-339} shows that stochastic optimsm is preserved under convex and increasing operations, we have
\begin{align*}
\max_{a \in \actions} \tilde{\theta}_{t, a} \succcurlyeq_{SO} \max_{a \in \actions} \theta_{t, a}. 
\end{align*}
By definition of stochastic optimism, 
we have
\begin{align*}
    \E\left[\max_{a \in \actions} \tilde{\theta}_a | \tilde{H}_t = h\right] \geq \E\left[\max_{a \in \actions} \theta_a | H_t = h\right]. 
\end{align*}
Hence, we've shown that for all $t \in \mathbb{Z}_+$, $h \in \mathcal{H}_t$, and $a \in \actions$, 
 	\begin{align*}
		\mathbb{E}\left[\tilde{R}_{*} \bbar  \tilde{H}_t = h\right] \geq \mathbb{E}[R_* | H_t = h].
	\end{align*}. 
\end{proof}
\subsection{An Example in Which the Gaussianity Assumption Holds: Proof of Lemma~\ref{lemma:gaussianity_condition}
}
\label{sec:gaussianity_condition}
\gaussianitycondition*
\begin{proof} 
First, $\E[\hat{\theta}] = \E[\tilde{\theta}] - \E[Z] = \E[\tilde{\theta}] = \mu_0$. In addition, since $\hat{\theta} \perp Z$ and $\Var(Z) = \delta^2 \Sigma_0$, so the covariance matrix of $\hat{\theta}$ is
\begin{align*}
    \Var(\hat{\theta}) = \Var(\tilde{\theta}) - \Var(Z) = \Sigma_0 - \delta^2 \Sigma_0 = 
\left(1 - \delta^2\right) \Sigma_0, 
\end{align*}
and the cross-covariance matrix of $\hat{\theta}$ and $\tilde{\theta}$ is 
\begin{align*}
    \mathrm{Cov}(\hat{\theta}, \tilde{\theta})
    = \mathrm{Cov}(\hat{\theta}, \hat{\theta} + Z)
    = \Var(\hat{\theta}) = \left(1 - \delta^2\right) \Sigma_0. 
\end{align*}
So the vector constructed by stacking $\hat{\theta}$ and $\tilde{\theta}$ is distributed according to the following multivariate Gaussian distribution:
\begin{align*}
   \begin{bmatrix}
		\hat{\theta} \\
		\tilde{\theta}
	\end{bmatrix}
	= \mathcal{N}\left(
    \begin{bmatrix}
		\mu_0 \\
		\mu_0
	\end{bmatrix}
	,
    \begin{bmatrix}
		\left(1 - \delta^2\right) \Sigma_0 & \left(1 - \delta^2\right) \Sigma_0 \\
		\left(1 - \delta^2\right) \Sigma_0 & \Sigma_0
	\end{bmatrix}
	\right).
\end{align*}
\end{proof}

\section{Example Regret Bounds}
\label{sec:reg_bds}

\subsection{Bounding the Information Ratio: Proof of Lemma~\ref{lemma:IR_bound}}
\label{sec:gaussian_ir_bound}
The following lemma bounds the information ratio defined with respect to the learning target $\chi = \hat{\theta}$.

\gaussianir*
\begin{proof} 
We show that for all $t \in \mathbb{Z}_+$ and $h \in \mathcal{H}_t$, the ratio
\begin{align*}
      \frac{\E\left[\tilde{R}_{*} - \tilde{R}_{t+1, \tilde{A}_t} - \epsilon \bbar  \tilde{H}_t = h\right]^2_{+}}{\I\left(\hat{\theta}; \tilde{A}_t, \tilde{R}_{t+1, \tilde{A}_t} \bbar  \tilde{H}_t = h\right)} \leq 2 \actions \sigma^2. 
\end{align*}
Fix $t \in \mathbb{Z}_+$ and $h \in \mathcal{H}_t$. If $\E\left[\tilde{R}_* - \tilde{R}_{t+1, \tilde{A}_t} | \tilde{H}_t = h\right] < \epsilon$, then the ratio is zero and 
the bound is trivially satisfied. Let us consider the case where $\E\left[\tilde{R}_* - \tilde{R}_{t+1, \tilde{A}_t} | \tilde{H}_t = h\right] \geq \epsilon$. Note that, for all $a, b \in \mathbb{R}_+$ with $a \geq b$, we have $(a-b)^2  \leq (a+b)(a-b) = a^2 - b^2$. Hence, 
\begin{align*}
\label{eq:gaussian_ir_proof_1}
\E\left[\tilde{R}_* - \tilde{R}_{t+1, \tilde{A}_t} - \epsilon | \tilde{H}_t = h\right]_+^2
= &\ \E\left[\tilde{R}_* - \tilde{R}_{t+1, \tilde{A}_t} - \epsilon | \tilde{H}_t = h\right]^2\\
\leq &\ \E\left[\tilde{R}_* - \tilde{R}_{t+1, \tilde{A}_t} | \tilde{H}_t = h\right]^2 - \epsilon^2. \numberthis
\end{align*}
It follows from Corollary~1 of \citep{russo2016information}(see Appendix~D.2) that for all $t \in \mathbb{Z}_+$, and $h \in \histories$,
\begin{align}
\label{eq:russo_cor1}
\E\left[\tilde{R}_* - \tilde{R}_{t+1, \tilde{A}_t} | \tilde{H}_t = h\right]^2 \leq 2 \actions \sigma^2 \I\left(\tilde{\theta}; \tilde{A}_t, \tilde{R}_{t+1, \tilde{A}_t} | \tilde{H}_t = h\right).
\end{align}
In addition, we introduce the following Lemma.
\begin{lemma}
\label{lemma:epsilon_bound}
For all $t \in \mathbb{Z}_+$ and $h \in \mathcal{H}_t$, we have
\begin{align*}
    \epsilon \geq \sqrt{2 \actions \sigma^2 \left[ \I\left(\tilde{\theta}; \tilde{A}_t, \tilde{R}_{t+1, \tilde{A}_t} | \tilde{H}_t = h\right) - \I\left(\hat{\theta}; \tilde{A}_t, \tilde{R}_{t+1, \tilde{A}_t} | \tilde{H}_t = h\right) \right]_+} \geq 0.
\end{align*}
\end{lemma}
Plug \eqref{eq:russo_cor1} into \eqref{eq:gaussian_ir_proof_1} and apply Lemma~\ref{lemma:epsilon_bound}, we have:
\begin{align*}
&\ \E\left[\tilde{R}_* - \tilde{R}_{t+1, \tilde{A}_t} - \epsilon | \tilde{H}_t = h\right]_+^2\\
\leq &\ \E\left[\tilde{R}_* - \tilde{R}_{t+1, \tilde{A}_t} | \tilde{H}_t = h\right]^2 - \epsilon^2 \\
\leq &\ 2\actions \sigma^2  \I\left(\tilde{\theta}; \tilde{A}_t, \tilde{R}_{t+1, \tilde{A}_t} | \tilde{H}_t = h\right) - 2 \actions \sigma^2 \left[\I\left(\tilde{\theta}; \tilde{A}_t, \tilde{R}_{t+1, \tilde{A}_t} | \tilde{H}_t = h\right) - \I\left(\hat{\theta}; \tilde{A}_t, \tilde{R}_{t+1, \tilde{A}_t} | \tilde{H}_t = h\right)\right]_+\\
\leq &\ 2 \actions \sigma^2  \I \left(\hat{\theta}; \tilde{A}_t, \tilde{R}_{t+1, \tilde{A}_t} | \tilde{H}_t = h\right).
\end{align*}
Hence, we've completed the proof of $\tilde{\Gamma}_{\hat{\theta}, \epsilon} \leq 2 \actions \sigma^2$. 
\end{proof}

\noindent \textbf{Proof of Lemma~\ref{lemma:epsilon_bound}}. 
\begin{proof}
By the data-processing inequality and the chain rule of mutual information (Lemmas~\ref{lemma:dp_MI} and \ref{lemma:chain_MI} in Appendix~\ref{sec:information}), we have for all $t \in \mathbb{Z}_+$ and $h \in \mathcal{H}_t$,
 \begin{align*}
    &\ \I\left({\tilde{\theta}}; \left(\tilde{A}_t, \tilde{R}_{t+1, \tilde{A}_t}\right) | \tilde{H}_t = h\right) - \I\left(\hat{\theta}; \left(\tilde{A}_t, \tilde{R}_{t+1, \tilde{A}_t}\right) | \tilde{H}_t = h\right) \\
    \overset{}{\leq} &\ \I\left({\hat{\theta}}, Z; \left(\tilde{A}_t, \tilde{R}_{t+1, \tilde{A}_t}\right) | \tilde{H}_t = h\right) - \I\left(\hat{\theta}; \left(\tilde{A}_t, \tilde{R}_{t+1, \tilde{A}_t}\right) | \tilde{H}_t = h\right) \\
    \overset{}{=} &\ \I\left(Z; \left(\tilde{A}_t, \tilde{R}_{t+1, \tilde{A}_t}\right) | \hat{\theta}, \tilde{H}_t = h\right). 
\end{align*}
Observe that conditioned on $\tilde{H}_t = h$, $\tilde{A}_t$ is independent of the rest of the system. Then we have for all $t \in \mathbb{Z}_+$ and $h \in \mathcal{H}_t$: 
    \begin{align*}
    \label{eq:gaussian_epsilon_proof_2}
    &\ \I\left({\tilde{\theta}}; \left(\tilde{A}_t, \tilde{R}_{t+1, \tilde{A}_t}\right) | \tilde{H}_t = h\right) - \I\left(\hat{\theta}; \left(\tilde{A}_t, \tilde{R}_{t+1, \tilde{A}_t}\right) | \tilde{H}_t = h\right) \\
    \leq &\ 
    \I\left(Z; \tilde{R}_{t+1, \tilde{A}_t} | \tilde{A}_t, \hat{\theta}, \tilde{H}_t = h\right) \\
    = &\ 
    \E\left[\I\left(Z; \tilde{R}_{t+1, \tilde{A}_t} | \tilde{A}_t = \tilde{A}_t, \hat{\theta} = \hat{\theta}, \tilde{H}_t = h\right)\right] \\
     = &\ \E\left[\sum_{a \in \actions} \Pr\left(\tilde{A}_t = a | \hat{\theta} = \hat{\theta}, \tilde{H}_t = h \right) \I\left(Z; \tilde{R}_{t+1, a} | \tilde{A}_t = a, \hat{\theta} = \hat{\theta}, \tilde{H}_t = h\right) \right]\\
    \leq &\ \E\left[\max_a \I\left(Z; \tilde{R}_{t+1, a} | \tilde{A}_t = a, \hat{\theta} = \hat{\theta}, \tilde{H}_t = h\right)\right] \\
    \overset{}{=} &\ \E\left[\max_a \I\left(Z; \tilde{R}_{t+1, a} | \hat{\theta} = \hat{\theta}, \tilde{H}_t = h\right) \right]\\
    = &\ \E\left[\max_a \I\left(Z; \tilde{R}_{t+1, a} - \hat{\theta}_{a} | \hat{\theta} = \hat{\theta}, \tilde{H}_t = h\right)\right] \\
    \overset{(a)}{=} &\ \E\left[\max_a \I\left(Z; Z_a + W_{t + 1,a} | \hat{\theta} = \hat{\theta}, \tilde{H}_t = h\right)\right] \\
     \overset{}{=} &\ \E\left[ \max_a \left\{ \I\left(Z_a; Z_a + W_{t + 1,a} | \hat{\theta} = \hat{\theta}, \tilde{H}_t = h\right) + \I\left(Z_{-a}; Z_a + W_{t + 1,a} | \hat{\theta} = \hat{\theta}, \tilde{H}_t = h\right) \right\}\right] \\
    \overset{}{=} &\ \E\left[ \max_a \I\left(Z_a; Z_a + W_{t + 1,a} | \hat{\theta} = \hat{\theta}, \tilde{H}_t = h\right)\right], \numberthis
    \end{align*}
where $(a)$ follows from the definition of the imaginary rewards: $\tilde{R}_{t+1,a} = \hat{\theta}_a + Z_a + W_{t+1,a}$, and that $Z_{-a}$ denotes the vector constructed by removing element $Z_a$ from the vector $Z$.  

For all $t \in \mathbb{Z}_+, h = (a_0, r_{1, a_0}, ... , a_{t-1}, r_{t, a_{t-1}}) \in \mathcal{H}_t$, 
$a \in \actions$, and $\underline{\theta} \in \R^\actions$,
     \begin{align*}
     &\ \Pr \left(Z_a \in \cdot |\hat{\theta} = \underline{\theta}, \tilde{H}_t = h\right) \\
      = &\ \Pr \left(Z_a \in \cdot |\hat{\theta} = \underline{\theta},  \tilde{R}_{1, a_0} = r_{1, a_0}, ... , \tilde{R}_{t, a_{t-1}} = r_{t, a_{t-1}}\right) \\
     = &\ \Pr \left(Z_a \in \cdot |\hat{\theta} = \underline{\theta}, Z_{a_0} + W_{1, a_0} = r_{1, a_0} - \underline{\theta}_{a_0}, ... , Z_{a_{t-1}} + W_{t, a_{t-1}} = r_{t, a_{t-1}} - \underline{\theta}_{a_{t-1}}\right) \\
     = &\ \Pr \left(Z_a \in \cdot | Z_{a_0} + W_{1, a_0} = r_{1, a_0} - \underline{\theta}_{a_0}, ... , Z_{a_{t-1}} + W_{t, a_{t-1}} = r_{t, a_{t-1}} - \underline{\theta}_{a_{t-1}}\right). 
\end{align*}
Since $Z_a$ and $ Z_{a_0} + W_{1, a_0}, ... , Z_{a_{t-1}} + W_{t, a_{t-1}}$ are jointly Gaussian, so the conditional distribution of $Z_a$ is Gaussian, and 
\begin{align*}
    \Var\left(Z_a | \hat{\theta} = \underline{\theta}, \tilde{H}_t = h\right) 
    = &\ \Var \left(Z_a | Z_{a_0} + W_{1, a_0} = r_{1, a_0} - \underline{\theta}_{a_0}, ... , Z_{a_{t-1}} + W_{t, a_{t-1}} = r_{t, a_{t-1}} - \underline{\theta}_{a_{t-1}}\right) \\ 
    \leq &\ \Var(Z_a) \\
    = &\ \delta^2 \Sigma_{0, a, a}. 
\end{align*}
By Lemma~\ref{lemma:gaussian_entropy},  
we have for all $t \in \mathbb{Z}_+$, $h \in \mathcal{H}_t$, $a \in \actions$, and $\underline{\theta} \in \R^{\actions}$, 
\begin{align*}
    \label{eq:gaussian_epsilon_proof_4}
    \I\left(Z_a; Z_a + W_{t + 1,a} |\hat{\theta} = \underline{\theta}, \tilde{H}_t = h\right) 
    \leq &\ \frac{1}{2} \ln\left(\delta^2 \Sigma_{0, a, a} + \sigma^2\right) - \frac{1}{2} \ln \sigma^2 \\
    = &\ \frac{1}{2} \ln\left(1 + \frac{\delta^2 \Sigma_{0, a, a}}{\sigma^2} \right) \\
    \leq &\ \frac{\delta^2 \Sigma_{0, a, a}}{2\sigma^2}, \numberthis
\end{align*}
where the last inequality follows from the fact that $\ln(1+x) \leq x$ for $x \geq 0$.

Combing \eqref{eq:gaussian_epsilon_proof_2} and \eqref{eq:gaussian_epsilon_proof_4}, we have for all $t \in \mathbb{Z}_+$ and $h \in \mathcal{H}_t$, 
\begin{align*}
    \I\left({\tilde{\theta}}; \left(\tilde{A}_t, \tilde{R}_{t+1, \tilde{A}_t}\right) | \tilde{H}_t = h\right) - \I\left(\hat{\theta}; \left(\tilde{A}_t, \tilde{R}_{t+1, \tilde{A}_t}\right) | \tilde{H}_t = h\right) \leq \frac{\delta^2}{2\sigma^2} \max_{a \in \actions} \Sigma_{0, a, a}.
\end{align*}
Hence, for all $t \in \mathbb{Z}_+$ and $h \in \mathcal{H}_t$, we have
\begin{align*}
     \sqrt{2 \actions \sigma^2 \left[ \I\left(\tilde{\theta}; \tilde{A}_t, \tilde{R}_{t+1, \tilde{A}_t} | \tilde{H}_t = h\right) - \I\left(\hat{\theta}; \tilde{A}_t, \tilde{R}_{t+1, \tilde{A}_t} | \tilde{H}_t = h\right) \right]_+} 
    \leq &\ \sqrt{2 \actions \sigma^2 \frac{\delta^2}{2 \sigma^2}} \\
    = &\ \delta \sqrt{\actions \max_{a \in \actions} \Sigma_{0, a, a}} \\
    = &\ \epsilon. 
\end{align*}
\end{proof}

\subsection{Regret Bounds: Proof of Theorem~\ref{th:Gaussian_regret}}
\label{sec:reg_bds}

\regretbounds*

\begin{proof}
Fix $\delta \in (0, 1)$. 
Let $\hat{\theta}$ be the learning target defined with respect to tolerance $\delta$, and let $\epsilon = \delta \sqrt{\actions \max_{a \in \actions}\Sigma_{0, a, a}}$. 
By Lemma~\ref{lemma:IR_bound}, 
$\tilde{\Gamma}_{\hat{\theta},\epsilon} 
\leq 2 \actions \sigma^2$. By Lemma~\ref{le:mi_theta_hat_theta}, $\I\left(\hat{\theta}; \pseudoenvironment \right)  = \frac{\actions}{2}  \ln \left(\frac{1}{\delta^2}\right) 
$. So 
\begin{align}
 \sqrt{\I\left(\proxytheta ; \pseudoenvironment \right) \tilde{\Gamma}_{\proxytheta, \epsilon} T} + \epsilon T  \leq  {\sigma} \actions \sqrt{T \ln \left( \frac{1}{\delta^2 } \right)} + \delta \sqrt{\actions \max_{a \in \actions} \Sigma_{0, a, a}} T.
\label{eq:regret_bound_1}
\end{align}

\begin{enumerate}
\item If $T > \actions$, let $\delta = \sqrt{\frac{\actions}{{T}}}$. Then
 \eqref{eq:regret_bound_1} becomes
\begin{align*}
\sqrt{\I\left(\proxytheta ; \pseudoenvironment \right) \tilde{\Gamma}_{\proxytheta, \epsilon} T} + \epsilon T 
& \leq {\sigma} \actions \sqrt{T \ln \left( \frac{T }{\actions} \right) } + \sqrt{\max_{a \in \actions} \Sigma_{0, a, a} \frac{\actions}{{T}}}  \sqrt{\actions} T 
\\
& = {\sigma} \actions \sqrt{T \ln \left( \frac{T }{\actions} \right) } + \actions \sqrt{T \max_{a \in \actions} \Sigma_{0, a, a}}.
\end{align*}
\item If $T \leq \actions$, fix $k \in \mathbb{Z}_+$ and let $\delta = 1 - \frac{1}{k}$. Then \eqref{eq:regret_bound_1} becomes
\begin{align*}
\sqrt{\I\left(\proxytheta ; \pseudoenvironment \right) \tilde{\Gamma}_{\proxytheta, \epsilon} T} + \epsilon T  
& \leq \sigma \actions \sqrt{T \ln\left(\frac{k^2}{(k - 1)^2}\right)} +  \left(1 - \frac{1}{k}\right)\sqrt{\max_{a \in \actions} \Sigma_{0, a, a} \actions}  T
\\
& \leq \sigma \actions \sqrt{T \ln\left(\frac{k^2}{(k - 1)^2}\right)} +  \left(1 - \frac{1}{k}\right) \actions \sqrt{ T \max_{a \in \actions} \Sigma_{0, a, a}}.
\end{align*}
Let $k \rightarrow +\infty$, we have 
\begin{align*}
\sqrt{\I\left(\proxytheta ; \pseudoenvironment \right) \tilde{\Gamma}_{\proxytheta, \epsilon} T} + \epsilon T  
 \leq \actions \sqrt{ T \max_{a \in \actions} \Sigma_{0, a, a}}.
\end{align*}
\end{enumerate}
Combining the two cases, we have shown that for all $T \in \mathbb{Z}_+$, there exists $\delta \in (0, 1)$ such that
\begin{align*}
\sqrt{\I\left(\hat{\theta}; \pseudoenvironment \right) \tilde{\Gamma}_{\hat{\theta},\epsilon}T} + \epsilon T \leq \sigma \actions \sqrt{T \ln \left(\frac{T }{\actions} \right) } + 
\actions \sqrt{T \max_{a \in \actions} \Sigma_{0, a, a}}. 
\end{align*}
\end{proof}

\end{document}